\documentclass[10pt,twocolumn,letterpaper]{article}

\usepackage{iccv}
\usepackage{times}
\usepackage{epsfig}
\usepackage{graphicx}
\usepackage{amsmath}
\usepackage{amssymb}

\usepackage{booktabs}
\usepackage{relsize}
\usepackage{setspace}
\usepackage{colortbl}
\usepackage{adjustbox}
\usepackage{mathtools}
\usepackage{algorithm} 
\usepackage{algpseudocode}
\usepackage{multirow}
\usepackage{color}
\usepackage{array}
\usepackage{cuted}
\usepackage{tikz}
\usepackage{multicol}

\usepackage{pythonhighlight}
\usepackage{tabularx}
\usepackage{derivative}
\usepackage{gensymb}
\usepackage{amsthm}
\usepackage{listings}

\usepackage[accsupp]{axessibility}

\usepackage{color}
\definecolor{amber}{rgb}{1.0, 0.75, 0.0}
\definecolor{green2}{rgb}{0.0, 0.5, 0.0}
\definecolor{dkgreen}{rgb}{0,0.6,0}
\definecolor{gray}{rgb}{0.5,0.5,0.5}
\definecolor{mauve}{rgb}{0.58,0,0.82}
\definecolor{beige}{HTML}{DBD7A2}
\usepackage[pagebackref=true,breaklinks=true,letterpaper=true,colorlinks,bookmarks=false]{hyperref}

\usepackage[capitalize]{cleveref}
\crefname{section}{Sec.}{Secs.}
\Crefname{section}{Section}{Sections}
\Crefname{table}{Table}{Tables}
\crefname{table}{Tab.}{Tabs.}

\iccvfinalcopy 

\def\iccvPaperID{11319} 

\ificcvfinal\fi

\begin{document}

\title{Semantic-Aware Implicit Template Learning via Part Deformation Consistency}

\newcommand*\samethanks[1][\value{footnote}]{\footnotemark[#1]}
\author{
Sihyeon Kim\hspace{0.6cm}
Minseok Joo\hspace{0.6cm}
Jaewon Lee\hspace{0.6cm} \\
Juyeon Ko\hspace{0.6cm}
Juhan Cha\hspace{0.6cm}
Hyunwoo J. Kim\thanks{Corresponding author.}\vspace{0.3cm} \\
Department of Computer Science and Engineering, Korea University \\
{\tt\small \{sh\_bs15, wlgkcjf87, 2j1ejyu, juyon98, hanchaa, hyunwoojkim\}@korea.ac.kr}}

\maketitle
\ificcvfinal\thispagestyle{empty}\fi

\begin{abstract}
Learning implicit templates as neural fields has recently shown impressive performance in unsupervised shape correspondence.
Despite the success, we observe current approaches, which solely rely on geometric information, often learn suboptimal deformation across generic object shapes, which have high structural variability.
In this paper, we highlight the importance of part deformation consistency and propose a semantic-aware implicit template learning framework to enable semantically plausible deformation.
By leveraging semantic prior from a self-supervised feature extractor, we suggest local conditioning with novel semantic-aware deformation code and deformation consistency regularizations regarding part deformation, global deformation, and global scaling.
Our extensive experiments demonstrate the superiority of the proposed method over baselines in various tasks: keypoint transfer, part label transfer, and texture transfer.
More interestingly, our framework shows a larger performance gain under more challenging settings.
We also provide qualitative analyses to validate the effectiveness of semantic-aware deformation.
The code is available at https://github.com/mlvlab/PDC.
\end{abstract}

\newcommand{\argminU}{\mathop{\mathrm{argmin}}}
\newcommand{\pdcgeo}{\mathcal{L}_{\text{pdc}\_\text{geo}}}
\newcommand{\pdcsem}{\mathcal{L}_{\text{pdc}\_\text{sem}}}
\newcommand{\bx}{\boldsymbol{x}}
\newcommand{\by}{\boldsymbol{y}}
\newcommand{\bo}{\boldsymbol{o}}
\newcommand{\bp}{\boldsymbol{p}}
\newcommand{\bq}{\boldsymbol{q}}
\newcommand{\bz}{\boldsymbol{z}}
\newcommand{\Pprime}{P^{\prime}}
\newcommand{\Qprime}{Q^{\prime}}
\newcommand{\Pprimei}{\mathcal{P}^{\prime}_i}
\newcommand{\Qprimei}{\mathcal{Q}^{\prime}_i}
\newcommand{\Ppset}{\Omega_{\Qprime}}
\newcommand{\Qpset}{\Omega_{\Qprime}}
\newcommand{\Ppseti}{\Omega_{\Pprime_i}}
\newcommand{\Qpseti}{\Omega_{\Qprime_i}}
\newcommand{\lrec}{\mathcal{L}_{\text{rec}} }
\newcommand{\gdc}{\mathcal{L}_{\text{geo}}}
\newcommand{\NF}{\mathcal{F}_\theta}
\newcommand{\dfm}{\Delta \boldsymbol{x}}
\newcommand{\cor}{\Delta s}
\newcommand{\Dc}{\mathcal{D}_{\theta_1}}
\newcommand{\OPx}{\boldsymbol{o}^P(\boldsymbol{x})}
\newcommand{\OPix}{\boldsymbol{o}^P_i(\boldsymbol{x})}
\newcommand{\OPjx}{\boldsymbol{o}^P_j(\boldsymbol{x})}
\newcommand{\zp}{\boldsymbol{z}_P}
\newcommand{\xb}{\boldsymbol{x}}
\newcommand{\eb}{\boldsymbol{e}}
\newcommand{\alphax}{\boldsymbol{\alpha}(\boldsymbol{x})}
\newcommand{\alphaxo}{\boldsymbol{\alpha}(\boldsymbol{x};P)}
\newcommand{\alphaxq}{\boldsymbol{\alpha}(\boldsymbol{x};Q)}
\newcommand{\Ec}{\mathcal{E}}
\newcommand{\Lc}{\mathcal{L}}
\newcommand{\Rb}{\mathbb{R}}
\newtheorem{lemma}{Lemma}

\newcommand*\circled[1]{\tikz[baseline=(char.base)]{
            \node[shape=circle,draw,inner sep=1pt] (char) {#1};}}
\newcommand{\sk}[1]{{\color{red}#1}}
\newcommand{\hjk}[1]{{\color{brown}#1}}
\begin{figure}[t!] 
\centering
\includegraphics[width=0.48\textwidth]{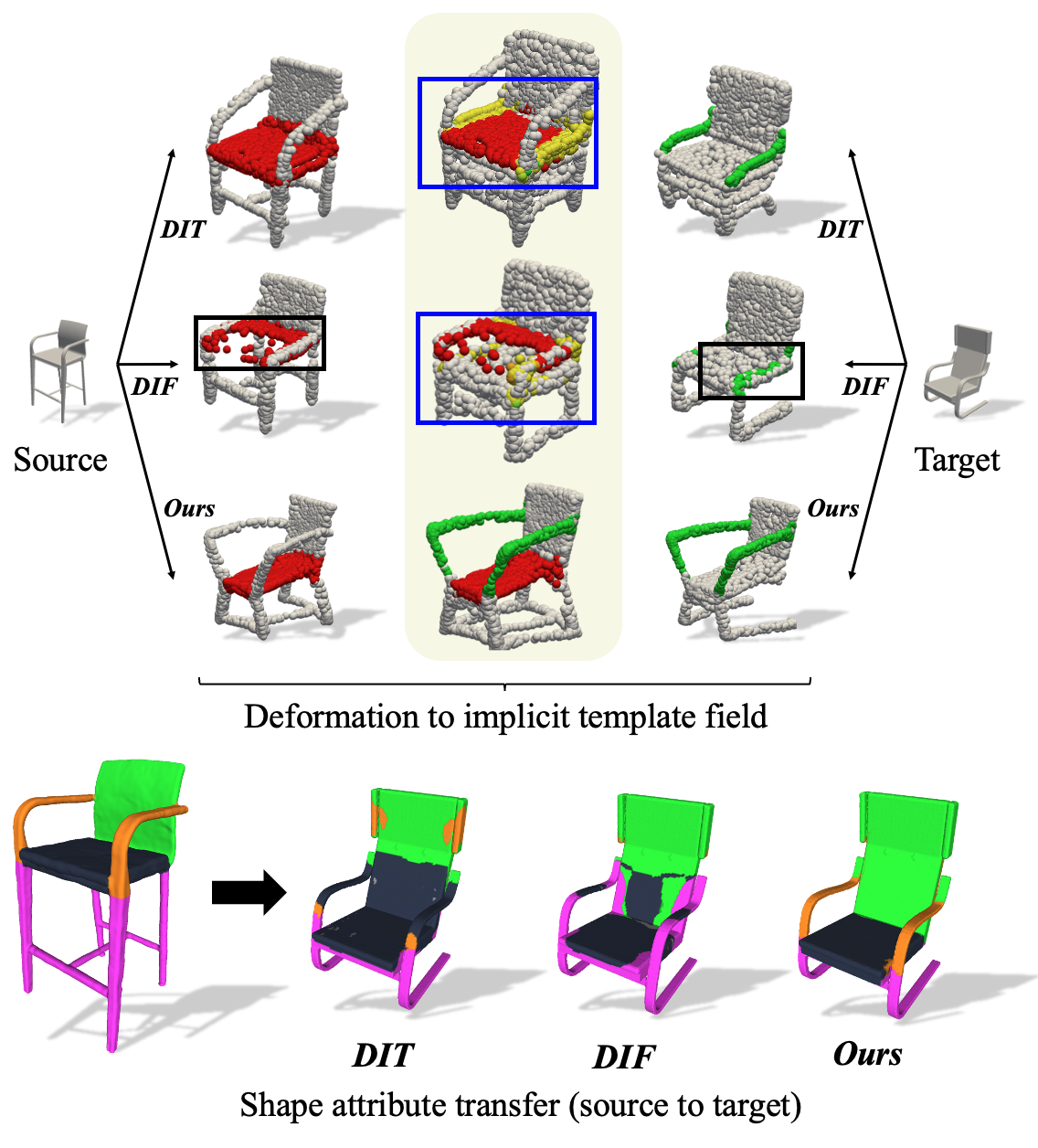}
\caption{ \textbf{Importance of part deformation consistency.} 
Given input shapes with high structural variability, only considering low-level geometry during template learning~\cite{zheng2020dit,deng2021deformed} can lead to suboptimal deformation (\eg, cannot distinguish between seat and arm where two parts are geometrically close while semantically distant: low seat-low arm, high seat-high arm). 
Ours succeeds in transferring shape attributes by encouraging part deformation consistency with semantic information.
}
\vspace{-16pt}
\label{fig:concept}
\end{figure}

\vspace{-15pt}
\section{Introduction}
Template learning is essential for shape analysis since the template, which is a compact representation of arbitrary shapes, provides a good prior for shape fitting so that we can establish dense correspondence and transfer key attributes such as texture or keypoints between shapes.
Traditionally, handcrafted templates such as meshes, part elements with primitive shapes ($\eg$, sphere, cylinder, cone), or parametric models have been widely used for human/animal body~\cite{DBLP:journals/tog/LoperM0PB15,allen2002articulated,3DCODED/eccv/GroueixFKRA18,joo2018total, DBLP:conf/cvpr/ZengOLLW20, lee2019dense,DBLP:conf/cvpr/ZuffiKJB17}, human face~\cite{paysan20093d, li2017learning}, human hand~\cite{ballan2012motion, sridhar2013interactive, oikonomidis2014evolutionary}, and generic objects~\cite{DBLP:conf/cvpr/TulsianiSGEM17, ganapathi2018parsing, groueix2018}. 
Defining templates for less complex shapes with consistent topologies and similar structures is relatively straightforward. 
However, for shapes like generic object shapes with high structural variability or non-rigid shapes with large-scale deformations, manually defining suitable templates becomes challenging.

To address this challenge, recent works~\cite{zheng2020dit,deng2021deformed,palafox2022spams,sundararaman2022implicit} have focused on learning implicit templates based on neural fields~\cite{xie2022neural}.
Neural fields, which are known for strong representation power, efficiency, and flexibility, have been successful in multiple 3D vision tasks, $\eg$, shape/scene reconstruction~\cite{chen2018implicit_decoder,mescheder2019occupancy,park2019deepsdf,peng2020convolutional}, neural rendering~\cite{mildenhall2020nerf,martin2021nerf,park2021nerfies}, and human digitization~\cite{saito2019pifu,yenamandra2021i3dmm,saito2020pifuhd}.
Likewise, neural field-based template learning has shown superior performance on dense correspondence between complex shapes.
Both~\cite{zheng2020dit} and~\cite{deng2021deformed} suggest decomposing the implicit representation of shapes into a deformation field and a template field. 
Here, the template is learned as a continuous function, which is expected to capture shared structures among the given shapes within a category ($\eg$, airplane, chair), while the deformation depends solely on the conditioned geometry of each shape ($\ie$, SDF value and surface normal).
We observe current approaches often learn \textit{suboptimal deformation} especially when a shape category has high structural variability like diverse local deformation scales and unseen/missing part structures; see Figure~\ref{fig:concept}.
That is, incorporating extra knowledge of part semantics is necessary to handle such cases.

In this paper, we propose a novel framework to learn semantically plausible deformation by distilling the knowledge of self-supervised feature extractor~\cite{chen2019bae_net} and employing it as a semantic prior to satisfy \textit{part deformation consistency} during the template learning process.
Part deformation consistency is our inductive bias that the deformation result within the same part should be consistent.
By leveraging semantic prior to understanding part arrangements of shapes, we suggest local conditioning with semantic-aware deformation code and carefully designed regularizations to encourage semantically corresponding deformation.
Semantic-aware deformation code is a point-wise soft assignment of part deformation priors, where part deformation priors are multiple latent codes to encode information on corresponding part deformation. 
We calculate assigned weights considering the part semantics of each point.
In addition, we propose input space regularization and latent space regularization to encourage part-level distances to be close, providing an explicit way to control the deformation field so that the model can learn flexible deformation and successfully learn common structures under diverse shapes.
Lastly, we suggest global scale consistency regularization to preserve the volume of the template against large-scale deformations.
To validate the effectiveness of our framework, we conduct extensive surrogate experiments reflecting different levels of correspondence such as keypoint transfer, part label transfer, and texture transfer, and demonstrate competitive performance over baselines with qualitative analyses.

Our \textbf{contributions} are summarized as follows: \textcircled{\raisebox{-0.9pt}{1}} We propose a new framework that learns global implicit template fields founded on part deformation consistency while understanding part semantics to enable semantically plausible deformation, \textcircled{\raisebox{-0.9pt}{2}} We impose hybrid conditioning by combining global conditioning of geometry and local conditioning of part semantics for flexible deformation, \textcircled{\raisebox{-0.9pt}{3}} We design novel regularizations to successfully manipulate the template fields, \textcircled{\raisebox{-0.9pt}{4}} We show the effectiveness of our framework with both qualitative and quantitative experiments.

\section{Related Works}

\noindent\textbf{Template learning for shape correspondence.}
Defining a template for a set of arbitrary shapes has been studied for a long time.
A high-quality template provides a good initialization for shape alignment and various applications in vision and graphics.
For shapes with consistent topologies and structures such as faces or bodies of humans and animals, handcrafted templates have been popular~\cite{DBLP:journals/tog/LoperM0PB15,3DCODED/eccv/GroueixFKRA18,DBLP:conf/cvpr/ZengOLLW20,lee2019dense,DBLP:conf/cvpr/ZuffiKJB17} and still widely used today.
On the other hand, it was more challenging for generic object shapes with inconsistent topology and diverse part structures.
Prior works suggest to predefine or learn local templates from data as in primitive parts~\cite{DBLP:conf/cvpr/TulsianiSGEM17} or part elements~\cite{kim2013learning,ganapathi2018parsing,sif/iccv/GenovaCVSFF19,groueix2018,atlasnetv2deprelle2019learning} to handle the inconsistent structure.
Recent works have tried to learn a canonical shape of generic objects through implicit template field~\cite{zheng2020dit,deng2021deformed,lei2022cadex,hui2022neural}, which are proven to be effective for representing shapes.
However, existing approaches mostly lack explicit consideration of semantic knowledge, which enables semantically plausible deformation so that the learned template captures a common structure even with the various structured shapes.

\noindent\textbf{Neural fields for 3D representations.}
Recent studies have shown that representing 3D shapes or scenes as implicit representations with a neural network has several advantages: powerful representation, flexibility, and memory efficiency.
These are called neural fields~\cite{xie2022neural}.
For instance, Occupancy networks~\cite{mescheder2019occupancy} proposed a method to represent 3D geometry as a continuous occupancy function. 
DeepSDF~\cite{park2019deepsdf} learns 3D complex surfaces with a neural network of auto-decoder framework for 3D modeling. 
Further, implicit representation has been introduced to various applications~\cite{hao2020dualsdf,DBLP:conf/eccv/Curriculum}.
To leverage the advantages, many applications based on neural fields have emerged such as neural rendering~\cite{mildenhall2020nerf,martin2021nerf,park2021nerfies}, human digitization~\cite{saito2019pifu,yenamandra2021i3dmm,saito2020pifuhd}, and generative modeling~\cite{DBLP:conf/iclr/GuL0T22,niemeyer2021giraffe,niemeyer2021campari,bautista2022gaudi}.
They have also been applied to learn the template of complex shapes as mentioned above.

\noindent\textbf{Utilizing self-supervised segmentation models.}
Recently, there has been a growing body of research on analyzing complex shapes in a self-supervised manner due to the high cost of annotation for 3D data.
Unsupervised semantic segmentation/co-segmentation models like BAE-Net~\cite{chen2019bae_net}, RIM-Net~\cite{niu2022rim}, AdaCoSeg~\cite{zhu2020adacoseg}, and PartSLIP~\cite{liu2023partslip} have demonstrated the ability to comprehend part semantics without ground truth labels.
For 2D and multi-modal tasks, using large-scale self-supervised segmentation models like DINO~\cite{caron2021emerging} as a feature extractor has been widely adopted, including pixelNeRF~\cite{yu2021pixelnerf}, Feature Field~\cite{kobayashi2022decomposing}, and ClipNeRF~\cite{wang2022clip}.
Inspired by these recent works, we utilize pretrained unsupervised part segmentation models to harness part semantic information for template learning.

\begin{figure*}[t!] 
\centering
\includegraphics[width=1.0\textwidth]{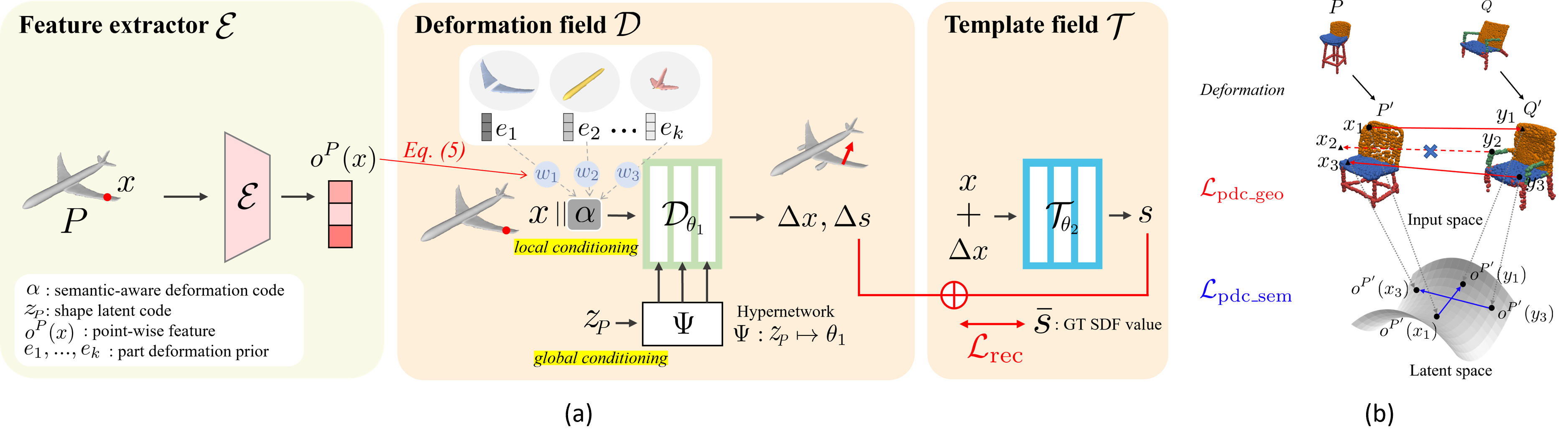}
\caption{\small \textbf{Framework of semantic-aware implicit template learning.} To learn semantic-aware implicit template field, (a) we extract part semantics from pretrained feature extractor $E$ and incorporate knowledge into semantic-aware deformation code (${\boldsymbol{\alpha}}$) and (b) part deformation consistency regularizations. With optimizing with the reconstruction loss $\mathcal{L}_{\text{rec}}$ and several regularizers, our framework succeeds to learn semantically plausible deformation.
\label{fig:model}
}
\vspace{-10pt}
\end{figure*}
\section{Preliminaries}\label{sec:3}

Implicit templates have been studied with \textit{neural fields}.
Given a 3D shape $P$, 
a neural field, which is a neural network $\NF$ parameterized by $\theta$, approximates 
continuous functions such as an occupancy function~\cite{chen2018implicit_decoder,mescheder2019occupancy} or a Signed Distance Function (SDF)~\cite{park2019deepsdf} as 
\begin{equation}
\NF (\boldsymbol{x};\boldsymbol{z}_P) = s.
\label{eq:deepsdf}
\end{equation}
Herein, $\boldsymbol{x} \in \mathbb{R}^{3}$ is a query point and $s$ is an occupancy/SDF function value. 
To represent multiple 3D shapes, for each shape, a shape latent code $\boldsymbol{z}_P\in \mathbb{R}^{d}$ is learned.
In general, a surface is defined as the zero-level set of a SDF, \ie, $\{ x| \NF (\boldsymbol{x};\boldsymbol{z}_P) = 0 \}$
as in DeepSDF~\cite{park2019deepsdf}.

The neural field can be further decomposed into an \textit{implicit template field} $\mathcal{T}$ and a deformation field $\mathcal{D}$~\cite{zheng2020dit, deng2021deformed}:
\begin{equation}
\NF(\bx;\boldsymbol{z}_P) = \mathcal{T}_{\theta_2}(\mathcal{D}_{\theta_1}(\boldsymbol{x};\boldsymbol{z}_P)).
\label{eq:git}
\end{equation}

The implicit template field $\mathcal{T}_{\theta_2}$ captures a common structure and outputs an occupancy/SDF value $s$ in the template space. 
$\mathcal{D}_{\theta_1}$ maps a point $x$ from a shape-specific space to the common template space by adding a deformation vector $\dfm \in \mathbb{R}^{3}$, \ie, $\bx + \dfm$.
In DIF-Net~\cite{deng2021deformed}, $\mathcal{D}_{\theta_1}$ additionally predicts a correction value $\cor \in \mathbb{R}$ to capture instance-specific local structures. 
In this case, the neural field $\NF$ is defined as 
\begin{equation}
\begin{split}
\NF(\bx;\boldsymbol{z}_P) &= \mathcal{T}_{\theta_2}(\boldsymbol{x}+\dfm)+\cor\\
\cor, \dfm &=\mathcal{D}_{\theta_1}(\boldsymbol{x};\boldsymbol{z}_P). 
\label{eq:DIFnet}
\end{split}
\end{equation}

The existing approaches~\cite{zheng2020dit, deng2021deformed}
learn the implicit template field by minimizing the distance between the template and the deformed individual shapes.
The distance is measured by SDF values and surface normals. 
Also, several regularizers are imposed such as deformation smoothness ($\eg$, penalizing the spatial gradient of $\mathcal{D}_{\theta_1}$ and restricting the deformation scale between points). 
However, the loss functions and regularizers are computed by \textit{low-level geometric} information. 
They are not effective enough to handle generic objects with large structural variability.
Hence, in this paper, we take \textit{part semantics} into account for template learning to achieve `semantically plausible' deformation and dense correspondence.

\section{Method}\label{section:method}
We present a semantic-aware implicit template learning framework based on part deformation consistency. 
In this section, we first delineate the overall architecture of our method equipped with part deformation priors and semantic-aware deformation code.
Then, we introduce our deformation consistency regularizers including part deformation consistency regularizations.
Lastly, we discuss the training procedure with several additional loss functions.

\subsection{Semantic-aware implicit template field}\label{sec:4.1}
We here present a novel semantic-aware implicit template learning framework. 
The overall pipeline of the proposed method is shown in Figure~\ref{fig:model}(a). 

\noindent\textbf{Template field and deformation field.}
We extend the implicit template learning by incorporating semantic information into the deformation field. 
Specifically, given a query point $\xb$ of a shape $P$, the deformation field $\Dc$ predicts a deformation vector $\dfm$
and a correction value $\cor$ conditioned on a point-wise \textit{semantic-aware deformation code} (SDC) $\alphaxo$ in addition to a shape latent code $\zp$. 
Then, our framework is defined as
\begin{equation}
\label{eq:ours2}
\begin{split} 
    \NF(\boldsymbol{x}; P, \mathcal{E}) &:= \mathcal{T}_{\theta_2}(\bx+\dfm)+\Delta s,\\
    \dfm, \Delta s  &= \mathcal{D}_{\theta_1}(\boldsymbol{x};\boldsymbol{z}_P, \alphaxo),
\end{split}
\end{equation}
where $\Ec$ is a pretrained feature extractor. 
We assign $\alphaxo$ based on a point-wise semantic feature, which is extracted from $\Ec$.
For simplicity, we denote the neural field $\NF(\boldsymbol{x}; P, \mathcal{E})$ by $\NF(\xb)$ and SDC $\alphaxo$ by $\alphax$.

\noindent\textbf{Feature extractor.}
In this paper, for feature extraction, we mainly use BAE-Net~\cite{chen2019bae_net} as an encoder $\Ec$, which is proven effective in capturing part semantics. 
Note that we pretrained BAE-Net on selected categories of ShapeNetV2~\cite{chang2015shapenet} by self-supervision without any extra labels or data for a fair comparison. 
The features from BAE-Net are $k$ dimensional vectors, \ie, $\boldsymbol{o}^P(\boldsymbol{x}) = \mathcal{E}(\bx, P)\in \mathbb{R}^k$.  
$k$ is a hyperparameter, which is closely related to the number of semantic parts. 
Since the number of semantic parts is unknown in self-supervised learning, we simply chose sufficiently large $k$.

\noindent\textbf{Semantic-aware deformation code.}
Our framework assumes that different semantic parts necessitate different deformations.
To achieve this, the proposed method learns a deformation prior for each semantic part. 
The deformation priors are encoded as latent codes $\boldsymbol{e}_1,\cdots,\boldsymbol{e}_k \in \mathbb{R}^{d^\prime}$. 
A point $\xb$ is softly assigned to $k$ semantic parts using the softmax values of semantic features $\OPx$. 
Combining the soft assignment with deformation priors, 
\textit{semantic-aware deformation code} (SDC) denoted as $\alphax \in \Rb^{d'}$ is defined as 
\begin{equation}
\alphax = \sum_{i=1}^k \frac{\text{exp}(\OPix)\eb_i}{\sum^k_{j=1} \text{exp}(\OPjx)} , \forall i \in \{1, \ldots, k\}, 
\end{equation}
where $\OPix$ denotes the $i$-th dimension value of the semantic features.
The SDC at $\xb$ is computed by a linear combination of deformation priors.
We learn the point's probability of belonging to each of the $k$ parts, which is served as part semantics.

\textit{Remarks.}
For comparison, we also applied hard assignment of part deformation priors for SDC, $\ie$, $\boldsymbol{\alpha}(\boldsymbol{x}) = \boldsymbol{e}_{i^*}$, where $i^* = {\mathrm{argmax}_i}(\boldsymbol{o}^P_i(\boldsymbol{x}))$.
However, as points are completely separated with local conditioning of the corresponding part deformation prior $\boldsymbol{e}_{i^*}$, the final deformation result tends to break the shape geometry (see Figure~\ref{fig:assignment}), leading to performance degradation as shown in Table~\ref{abl_assign}.
Thereby, we choose the soft assignment to enable flexible deformation while preserving the shape geometry. 

We train the deformation field $\mathcal{D}_{\theta_1}$ and $\mathcal{T}_{\theta_2}$ in an end-to-end fashion.
In our experiments, $\mathcal{D}_{\theta_1}$ is a 6-layer MLP conditioned on both \textit{global geometry} $\boldsymbol{z}_P$ via hypernetwork $\Psi$~\cite{hypernet/iclr/HaDL17,sitzmann2019siren} and \textit{local part semantics} $\boldsymbol{\alpha}(\boldsymbol{x})$ via concatenation with input coordinates.
The shape latent code $\boldsymbol{z}_P$ and part deformation priors $\{\boldsymbol{e}_i\}$ are randomly initialized and jointly optimized with neural fields.
$\mathcal{T}_{\theta_2}$ is a 5-layer MLP.

\begin{table}[ht!] 
  \centering
  \caption{\textbf{Ablations of different assignment strategies for SDC.} We conduct a quantitative comparison between hard and soft assignments with 5-shot part label transfer on ShapeNet chair.}
  \label{abl_assign}
\setlength{\tabcolsep}{7pt}
\renewcommand{\arraystretch}{1.}
\small
  \begin{tabular}{c|cc}
    \toprule
\multirow{1}{*}{\begin{tabular}[c]{@{}c@{}}Assignment strategies\end{tabular}} & \multirow{1}{*}{Hard} & \multicolumn{1}{c}{Soft}               \\ 
     \midrule
     chair (5-shot PLT)&77.0&\textbf{84.2}\\
    \bottomrule
  \end{tabular}
\end{table} 

\begin{figure}[t] 
\centering
\includegraphics[width=0.4\textwidth]{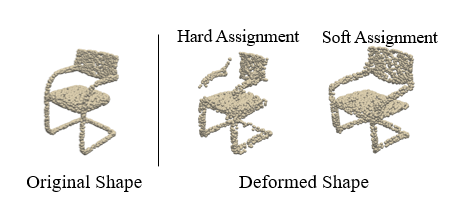}
\caption{\textbf{Effects of different assignment strategies for SDC.} Hard assignment breaks the shape geometry, while soft assignment preserves the geometry while learning flexible deformations.}
\vspace{-10pt}
\label{fig:assignment}
\end{figure}

\subsection{{Part deformation consistency regularizations}}\label{sec:4.2}
In addition to conditioning neural fields, 
leveraging semantic information, we propose novel part deformation consistency regularizations to encourage semantically plausible deformation.
Semantically similar points should be close after deformation in both an input space and a latent space. 
That is, the deformation of points from the same part should be consistent.
Figure~\ref{fig:model}(b) illustrates the concept of our regularizations.
Given a pair of deformed shapes $\Pprime,\Qprime$, part deformation consistency regularizations are defined as
\begin{align}
    \label{eq:pdc_geo}
    \nonumber \pdcgeo =& \sum^k_{i=1}\left(\frac{1}{\left|\Pprimei\right|} \sum_{\bx \in \Pprimei} \|\bx-C(\bx,\Qprimei)\|_{2}^{2} \right.\\ 
    +&\left. \frac{1}{\left|\Qprimei\right|} \sum_{\by \in \Qprimei} \|\by - C(\by,\Pprimei)\|_{2}^{2} \right), \text{ and }
\end{align}
\begin{align}  
    \label{eq:pdc_sem}
    \nonumber \pdcsem  =& \sum^k_{i=1}\left(\frac{1}{\left|\Pprimei\right|} \sum_{\bx \in \Pprimei} \|\bo^{\Pprime}(\bx) - \bo^{\Qprime}(C(\bx,\Qprimei))\|_{2}^{2} \right. \\ 
    + & \left. \frac{1}{\left|\Qprimei\right|} \sum_{\by \in \Qprimei} \|\bo^{\Qprime}(\by) - \bo^{\Pprime}(C(\by,\Pprimei))\|_{2}^{2}\right), 
\end{align}

\begin{figure*}[t!]
\centering
            \includegraphics[width=1.\textwidth]{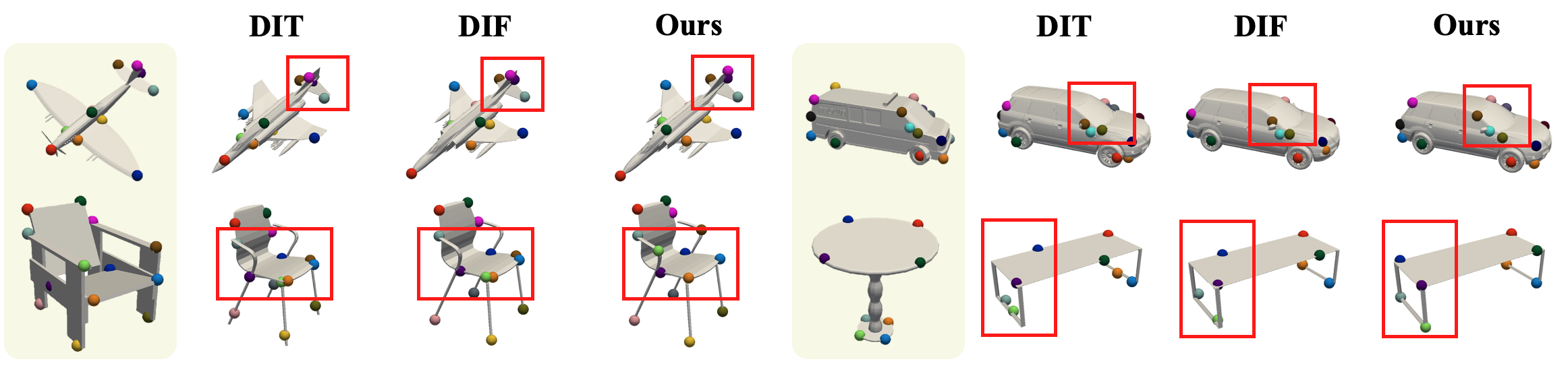}
        \caption
         {\textbf{Keypoint transfer result on ShapeNet.} We transfer keypoint labels from source shapes (colored in ivory) to target shapes and compare ours with DIT~\cite{zheng2020dit} and DIF~\cite{deng2021deformed}. For additional visualizations including results on ScanObjectNN~\cite{uy-scanobjectnn-iccv19}, see the supplement.}
\label{fig:kp}
\vspace{7pt}
\end{figure*}

\begin{figure}[t!] 

\centering
\includegraphics[width=0.45\textwidth]{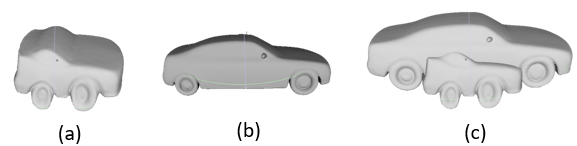}
\caption{\textbf{Efficacy of global scaling consistency $\mathcal{L}_{\text{scale}}$.}}
\label{fig:scvs}
\vspace{-10pt}
\end{figure}

\noindent where $C$ is a correspondence matching function for a point $\bx$ and its nearest point in a set $\mathcal{Q}$, \textit{i.e.}, $C(\bx,\mathcal{Q}) = \underset{\by\in \mathcal{Q}}{\mathrm{argmin}}\|\bx-\by\|^2_2$.
$\Pprimei$ and $\Qprimei$ denote sampled point sets of $i^{\text{th}}$ deformed part shapes of $\Pprime$ and $\Qprime$.  
They are defined as $\Pprimei = \{ \boldsymbol{x}+\mathcal{D}_{\theta_1}(\boldsymbol{x}; \boldsymbol{z}_P,\alphaxo) \, |\, i\!\!=\!\!\underset{k}{\mathrm{argmax}}(\boldsymbol{o}^{P}_{k}(\boldsymbol{x})), \boldsymbol{x} \!\in\!\Omega_P\}$ and $\Qprimei = \{ \boldsymbol{x}+\mathcal{D}_{\theta_1}(\boldsymbol{x}; \boldsymbol{z}_Q,\alphaxq) \, |\, i\!\!=\!\!\underset{k}{\mathrm{argmax}}(\boldsymbol{o}^{Q}_{k}(\boldsymbol{x})), \boldsymbol{x} \!\in\!\Omega_Q\}$, where $\Omega_P$ and $\Omega_Q$ are sampled query point sets for shape $P$ and $Q$.
Note that each deformed part shape is extracted based on part semantics from \textit{original} shapes $P$ and $Q$.
Here, we explicitly guide the part-level deformation among various shapes to be consistent.
Moreover, since we understand the part arrangement and only penalize matched parts, our method does not forcibly deform “unseen/missing match” cases.

We partly observed spatial distortion when part deformation consistency was too strongly imposed.
This is because using external semantics for deformation could break the continuity of geometric fields.
In addition, we propose \textit{global scale consistency} and \textit{global deformation consistency} to address the concern.
For the former one, we disentangle a scaling operation from the entire non-rigid deformation $\Dc$ to calculate the global scale of the learned implicit template field.
That is, we want to approximate $\mathcal{D}$ with the global scaling of the scaling factor of $r$, $\ie$, $\mathcal{D} \approx rI$.
Then the problem becomes finding $r$ from
\begin{equation}
 \min_{r} \sum\limits_{s} \sum\limits_{i} \|\bx_i^s +\dfm_i^s - r\bx_i^s\|_2^2,
\label{eq:scale}
\end{equation} 
where $s$ and $i$ are an index for shape and point, respectively.
A closed-form solution of \eqref{eq:scale} is $r=\sum\limits_{s} \sum\limits_{i}\frac{\bx_i^{s^\top} (\bx_i^s +\dfm_i^s)}{\bx_i^{s^\top} \bx_i^s}$.
See the supplement for proof.
Since we want the learned $\mathcal{D}$ to be insensitive to global scale operation, $\ie$, the implicit template to keep its scale by the average scale of input shapes, we penalize the expected $\mathbb{E}[r]$ over input shapes $P_1\cdots P_N$ in a single batch as:
\begin{equation}
 \mathcal{L}_{\text{scale}}= \big| \mathbb{E}[r] -1 \big|
\label{eq:sclreg}
\end{equation} 
\Cref{fig:scvs} illuminates the effect of $\mathcal{L}_{\text{scale}}$. 
For the same input shapes, \Cref{fig:scvs}(a) is a car template learned without $\mathcal{L}_{\text{scale}}$, \Cref{fig:scvs}(b) is a car template learned with $\mathcal{L}_{\text{scale}}$, and \Cref{fig:scvs}(c) is the result of placing two shapes in the same space.
It shows $\mathcal{L}_{\text{scale}}$ helps the model to learn a more desirable template by penalizing the change in the global scale of the learned template fields.

For the latter one, we add global deformation consistency for a pair of deformed shapes $\Pprime,\Qprime$:
\begin{equation}
    \label{eq:gdc}
    \gdc = \textbf{CD}(\Pprime,\Qprime),
\end{equation}
where $\textbf{CD}$ denotes a conventional Chamfer loss.
Thus, the deformation consistency regularization can be written as:
\begin{equation}
    \label{eq:dc}
    \mathcal{L}_{\text{dc}} =  \gamma_1 \pdcgeo+\gamma_2 \pdcsem+\gamma_3 \mathcal{L}_{\text{scale}}+\gamma_4 \Lc_{\text{geo}}.
\end{equation}

\begin{figure*}[t!]
\centering
            \includegraphics[width=0.98\textwidth]{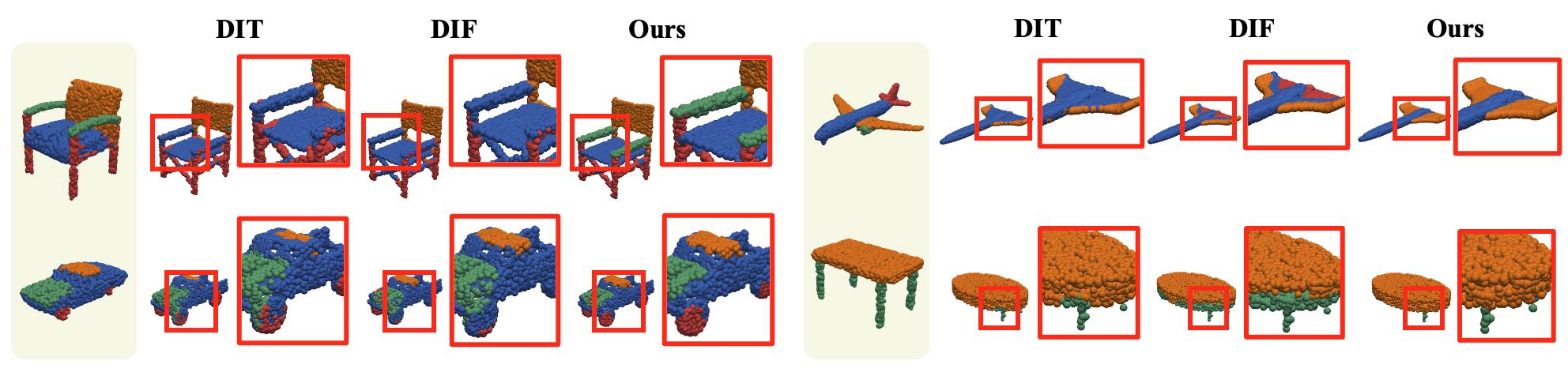}   
        \caption
         {\textbf{Part label transfer result on ShapeNet.} We transfer part segment labels from source shapes (colored in ivory) to target shapes and compare ours with DIT~\cite{zheng2020dit} and DIF~\cite{deng2021deformed}. For additional visualizations, please refer to the supplement.}
\label{fig:ptl}
\vspace{-5pt}
\end{figure*}

\subsection{Optimization}\label{sec:4.3}
We train our deformation and template fields with a reconstruction loss and several regularizers including our deformation consistency regularization.
Following~\cite{deng2021deformed}, the reconstruction loss $\lrec$ for shape $P$ is
\begin{equation}
\label{eq:recon}
\scalebox{0.85}{
$\begin{aligned}    
    \mathcal{L}_{\text{rec}}&= 
    \sum\limits_{\bx\in\Omega_P}\big| \NF(\bx) - \bar{s}\big| + 
    \sum\limits_{\bx\in P}(1-\langle\nabla\NF(\bx),\bar{n}\rangle)\\ 
    &+\sum\limits_{\bx\in\Omega_P} \big|\|\nabla\NF(\bx)\|_2 - 1\big|  + \sum\limits_{\bx\in\Omega_P\backslash P}
    {\rm exp}(-\delta \cdot \big|\NF(\bx)\big|),
\end{aligned}$}
\end{equation}  
where $\bar{s}$ denotes the ground-truth SDF value of each query point, $\bar{n}$ is the ground truth surface normal vector of each surface point and $\nabla \mathcal{F}_\theta$ the spatial derivative of a 3D field.
The second term of ~\eqref{eq:recon} is derived from the property of the SDF function that the gradient of an SDF is identical to the surface normal for surface points and the third is implicit geometric regularization~\cite{DBLP:conf/icml/IGR} derived from the Eikonal equation. 
The last term pertains to penalizing off-surface points to be away from the surface points, which is controlled by $\delta \gg 1$.

Since we do not know ground truth mapping between the template and input shapes, we manipulate the deformation field through multiple regularizers. 
First, we employ our proposed deformation consistency regularization~\eqref{eq:dc}.
For further regularizers, we adopt~\cite{deng2021deformed} such as deformation smoothness $\mathcal{L}_{\text{smooth}}$, normal consistency $\mathcal{L}_{\text {normal}}$ and minimal correction $\mathcal{L}_{\text{c}}$, which are defined as
\begin{equation}
    \mathcal{L}_{\text{smooth}} =  \sum\limits_{\bx\in\Omega_P}\sum_{d\in\{X,Y,Z\}}\|\nabla \mathcal{D}_{\theta_1}^v(\,\cdot,\boldsymbol{z}_P)|_d(\bx)\|^2_2, \label{eq:smooth}
\end{equation}
\begin{equation}
\mathcal{L}_{\text {normal}}=\sum_i \sum_{ \bx\in \mathcal{S}_i}\left(1-\left\langle\nabla \mathcal{T}_{\theta_2}\left(\bx+D_{\theta_1}^v(\bx)\right), \bar{n}\right\rangle\right), \label{eq:normal}
\end{equation}
\begin{equation}
    \mathcal{L}_{\text{c}} = \sum_i\sum\limits_{\bx\in\Omega_P} | \mathcal{D}_{\theta_1}^{\Delta s}(\bx)|. \label{eq:correct}
\end{equation}
We also apply $\ell_2$ regularization on shape latent codes and part deformation priors
\begin{equation}
    \mathcal{L}_{\text{emb}} = \sum_i^N\|\boldsymbol{z}_i\|^2_2+\sum_j^k\|\boldsymbol{e}_j\|^2_2. \label{loss:emb}
\end{equation}
The regularizers are collectively written as
\begin{equation}
    \label{eq:reg}
    \mathcal{L}_{\text{reg}} =  \gamma_5 \mathcal{L}_{\text{smooth}}+\gamma_6 \mathcal{L}_{\text{normal}}+ \gamma_7 \mathcal{L}_{\text{c}}+ \gamma_8 \mathcal{L}_{\text{emb}}.
\end{equation}
Hence, the final loss becomes
\begin{equation}\label{eq:all}
    \mathcal{L} = \mathcal{L}_{\text{rec}}+ \mathcal{L}_{\text{dc}}+\mathcal{L}_{\text{reg}}.
\end{equation}

\noindent\textbf{Evaluating dense correspondence.}
Given two pairs of seen or unseen shapes $P, \,Q$, we can evaluate the dense correspondence between $P$ and $Q$ by deforming each shape to the common template space.
Moreover, the uncertainty of dense correspondence performance between two points $\bp\in P$ and $\bq\in Q$ can be measured as below:
\begin{equation}
\scalebox{1.}{
$u(p,q) = 1 - {\rm exp} (-\gamma\|(p+\Delta p)-(q+\Delta q)\|_2^2). \label{eq:uncertain}
$}
\end{equation} 
$\gamma$ is a hyperparameter for modulation and $\Delta p, \Delta q$ are output from $\mathcal{D}_{\theta_1}$.
Note that since each part deformation prior $\boldsymbol{e}$ is shared based on part semantics across entire training shapes, we do not need a test-time optimization on $\boldsymbol{e}$ for an unseen shape at inference time.

\begin{table}
	\caption{\textbf{Keypoint
 correspondence accuracy in KeypointNet~\cite{DBLP:conf/cvpr/YouLLCLMLW20}}. PCK scores with threshold of $0.05$ / $0.1$ are reported. }
	\label{tab:kp}
    \begin{adjustbox}{width=0.48\textwidth}
    \centering
	\begin{tabular}{l c c c c}
		\toprule
        {PCK} & {airplane} & {car} & {chair} & {table} \\
		\midrule	   
        {AtlasNetV2}~\cite{atlasnetv2deprelle2019learning}  &23.9 / 45.3 & 33.5 / 64.0 & 25.6 / 50.2 &25.3 / 48.0\\
        {DIT}~\cite{deng2021deformed}  &13.6 / 34.2 &58.9 / 87.7 & 33.3 / 66.7 & 26.2 / 55.5\\
        {DIF}~\cite{zheng2020dit}  &16.5 / 40.5 &67.8 / 90.4 & 34.3 / 57.6 & 37.9 / 62.9\\
		\midrule
		{Ours}  &\textbf{24.2} / \textbf{52.3} &\textbf{69.3} / \textbf{91.8} & \textbf{61.2} / \textbf{83.2} & \textbf{50.8}/\textbf{77.9}\\
		\bottomrule
	\end{tabular}
    \end{adjustbox}
\end{table}

\begin{table} 
	\caption{\textbf{5-shot part label transfer in four categories of ShapeNet}. The reported numbers are overall mIoU.}
	\label{tab:ptl}
    \centering
	\begin{tabular}{l c c c c}
		\toprule
        {model } & {airplane} & {car} & {chair} & {table} \\
		\midrule
            AtlasNetV2~\cite{atlasnetv2deprelle2019learning}&57.8&59.3  &66.9&70.6\\
		DIT~\cite{zheng2020dit}  &64.2 &66.1& 80.4 & 86.0  \\
		DIF~\cite{deng2021deformed}  &73.9&69.6&82.1 &86.6 \\
		\midrule
		Ours   &\textbf{75.1}&\textbf{70.4}&\textbf{84.2}  &\textbf{86.9} \\
		\bottomrule
	\end{tabular}

\end{table}

\section{Experiments}

We demonstrate the effectiveness of our proposed framework and provide both quantitative and qualitative results.
In \Cref{sec:5.1}, we conduct multiple label transfer tasks for dense correspondence evaluation.
In \Cref{sec:5.2}, we measure shape reconstruction accuracy to evaluate the representational power for unknown shapes.
In \Cref{sec:5.3}, we present advanced dense correspondence analysis under more challenging settings.
Lastly, in \Cref{sec:5.4}, we perform ablation studies to analyze the proposed regularization and the sensitivity over feature extractor~\cite{chen2019bae_net}.\\

\begin{figure*}[t!]
\centering
            \includegraphics[width=0.98\textwidth]{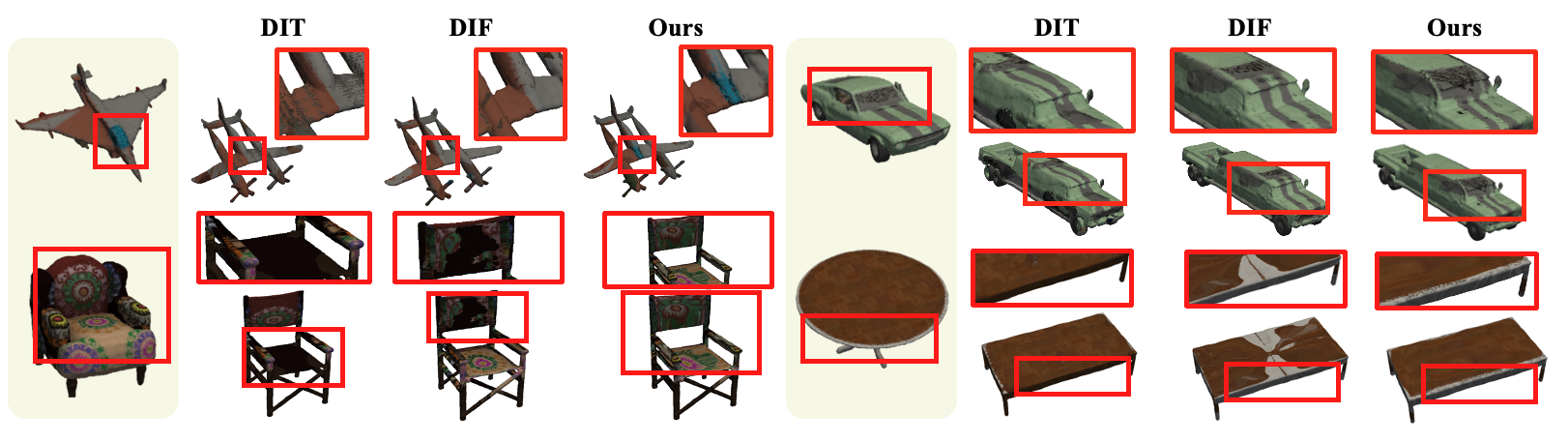}   
        \caption
         {\textbf{Texture transfer result on ShapeNet.} We transfer part segment labels from source shapes (colored in ivory) to target shapes and compare ours with DIT~\cite{zheng2020dit} and DIF~\cite{deng2021deformed}. For additional visualizations, please refer to the supplement.}
\label{fig:texture}
\vspace{-10pt}
\end{figure*}

\vspace{-10pt}
\noindent\textbf{Datasets and Baselines.}
We train and validate our model and baselines on four categories ($\eg$, chair, table, airplane and car) of ShapeNetV2~\cite{chang2015shapenet}.
We follow the training split of DIF~\cite{deng2021deformed} for each category.
For label transfer tasks, we use the labels from ShapeNet-Part~\cite{shapenetprat/tog/YiKCSYSLHSG16} and KeypointNet~\cite{DBLP:conf/cvpr/YouLLCLMLW20}.
DIT~\cite{zheng2020dit} and DIF~\cite{deng2021deformed} are the most relevant baseline models as they both globally learn implicit template fields, and we select AtlasNetV2~\cite{atlasnetv2deprelle2019learning} as an additional baseline.
Dense correspondence is measured with local patch composition in~\cite{atlasnetv2deprelle2019learning}.
For more implementation details, analyses, and visualizations, see the supplement.

\begin{table}[t]
	\caption{\textbf{Reconstruction accuracy for unseen shapes in four categories of ShapeNet.} The reported numbers are chamfer distance (CD) between ground truth and reconstructed meshes at a resolution of $256^3$.}
	\label{tab:recon}
    \centering
	\begin{tabular}{l c c c c}
		\toprule
        {CD ($\times10^3$)$\downarrow$ } & {airplane} & {car} & {chair} & {table} \\
		\midrule
		DIT~\cite{zheng2020dit}    &0.2261     & 0.4082&0.3063 &0.6187  \\
		DIF~\cite{deng2021deformed}   &0.2093   & 0.3157 &0.4229 & \textbf{0.2640}\\
		\midrule
		Ours     & \textbf{0.1859}   & \textbf{0.3150}  & \textbf{0.2265} & 0.2796  \\
		\bottomrule
	\end{tabular}
 \vspace{-5pt}
\end{table}

\subsection{Dense correspondence}\label{sec:5.1}
To evaluate the quality of learned dense correspondence, we conduct various surrogate tasks including keypoint transfer, part label transfer, and texture transfer due to the lack of labeled datasets for 3D dense correspondence.
For the following tasks, we transfer respective labels from source shapes to target shapes in a common template space.

\noindent\textbf{Keypoint label transfer.}
Table~\ref{tab:kp} presents the quantitative result in a 5-shot keypoint transfer task.
We measure the performance by a percentage of correct keypoints (PCK)~\cite{yi2016syncspeccnn}. 
We compute the distance between the transferred keypoints and the ground truth points to record the PCK score with the threshold of 0.05/0.1. 
Our method achieves the best performance in all four categories with a significant improvement. 
Especially for categories with high structural variability (\eg, chair and table), a large performance gain is observed.
For instance, even compared to the best baseline, on the chair category with the threshold of 0.1, our method (83.2) outperforms DIT~\cite{zheng2020dit} (66.7) by 16.5$\%$, and on the table category with the threshold of 0.1, our method (77.9) outperforms DIF~\cite{deng2021deformed} (62.9) by 15$\%$.
We additionally provide the visualization of keypoint transfer in Figure~\ref{fig:kp}, which also support the importance of understanding part semantics for shape correspondence.
In particular, two keypoints of pea-green and orange from the source chair are geometrically close but semantically different since one is for the arm part and the other is for the seat part.
Among several methods, only our framework succeeds in disentangling two keypoints in the target chair.
We provide more qualitative results regarding keypoint transfer results in real-world data~\cite{uy-scanobjectnn-iccv19} to validate the robustness of our method in the supplement.
\vspace{5pt}

\noindent\textbf{Part label transfer.}
For few-shot part label transfer tasks, we transfer part segment labels from source shapes to target shapes.
In detail, for each (target) point of the deformed target shape, we find the $n$ nearest points of the deformed source shapes and 
perform weighted voting using their labels and distances to the target point. 
In our experiments, we used $n=10$ nearest points.
For a 5-shot part label transfer task in Table~\ref{tab:ptl}, five source shapes were used just the same as~\cite{deng2021deformed} for every category.
Table~\ref{tab:ptl} shows our framework achieves the best performance over all four categories. 
Moreover, Figure~\ref{fig:ptl} illustrates the qualitative results of one-shot part label transfer on ShapeNet~\cite{chang2015shapenet}.
Similar to keypoint transfer, Figure~\ref{fig:ptl} also proves the effectiveness of semantic knowledge in dense correspondence.

\noindent\textbf{Texture transfer.}
Figure~\ref{fig:texture} shows texture transfer results on ShapeNetV2~\cite{chang2015shapenet} to compare established dense correspondence between baselines and ours with higher-frequency label transfer.
We observe notable texture transfer results.
For instance, our framework successfully transfers the source car's texture of the windshield ($\color{gray}{gray}$) to the corresponding part of the target cars, whereas baselines transfer the same textures to the roof part of the target cars. 

\subsection{Shape reconstruction}\label{sec:5.2}
We report the reconstruction quality of unseen shapes with Chamfer Distance (CD) to validate the representational power of our framework in Table~\ref{tab:recon}.
A smaller CD means more accurate reconstruction.  
As in previous work, we report the CD $\times 10^3$ for readability. 
Also, we consistently discuss the CD and performance gain below at this scale.
We optimize the corresponding shape latent code for a given unseen shape.
We measure CD as in DIF~\cite{deng2021deformed}, where reconstructed meshes are extracted at a resolution of $256^3$, and CD is calculated with $10K$ sampled points.
Our framework on average achieved the best reconstruction results surpassing DIF by a margin of 0.0512. 
On average, our method and DIF exhibit 0.2518 and 0.3030 reconstruction errors, respectively.
In three out of four categories, our method achieved the best reconstruction accuracy.

\subsection{More challenging setups}\label{sec:5.3}
To further validate the effectiveness of our method, we perform additional experiments and compare the results with DIF~\cite{deng2021deformed}.
Overall, the results below imply current methods that lack part semantics and do not consider part deformation consistency may fail to learn semantically plausible deformation in challenging cases.

\noindent\textbf{2-shot part label transfer within different subcategories.}
Our objective is to assess dense correspondence performance under difficult settings.
We conduct 2-shot part label transfer within \textit{two} different subcategories belonging to the same category, where shapes from different subcategories tend to have different geometry for the same semantics.
Since ShapeNet~\cite{chang2015shapenet} includes incorrect subcategory labels~\cite{liu2021fine}, we manually cleaned subcategory labels for four categories (\eg, airplane: airliner/jet, car: sedan/jeep, chair: straight chair/sofa, table: pedestal table/short table).
Figure~\ref{fig:source} displays examples of shapes from each subcategory.
For robust evaluations, we randomly select the source shape 5 times for each subcategory.
Table~\ref{tab:subcategory} shows our framework achieves a significant performance gain in all cases by over 7.2$\%$ on average compared to DIF~\cite{deng2021deformed}.

\noindent\textbf{Unsupervised dense correspondence in DFAUST~\cite{dfaust:CVPR:2017}.}
Furthermore, we evaluate unsupervised dense correspondence with non-rigid shapes~\cite{dfaust:CVPR:2017}.
For more challenging settings, we train both DIF~\cite{deng2021deformed} and our model for two different subjects with different sizes and heights  (\eg, subject ID of 50002, 50026) and four distinct actions (\eg, shake arms, chicken wings, running on spot, and jumping jacks). 
Colormap in Figure~\ref{fig:dfaust2} demonstrates the correspondence quality, where we transfer color over the up-axis.
In the source shape, the arms are labeled in red. 
After transferring these color labels to the target shapes, our method preserves the red/orange labeling of the arms across different motions, whereas the baseline fails to capture the same labeling (\ie, yellow hands) since it lacks semantic knowledge during training.
Additionally, we conduct keypoint transfer tasks and measure the PCK score as identical as in Section~\ref{sec:5.1}.
Our model achieves 74.9/85.2 with a threshold of 0.05/0.1, whereas DIF~\cite{deng2021deformed} achieves 67.0/80.7 with the same threshold, showing our superior performance in the challenging settings (see the supplement for the full keypoint transfer result).
Note that we can also learn implicit templates for non-rigid shapes with supervision as in~\cite{sundararaman2022implicit}.
We conduct experiments in unsupervised settings since our focus is to validate the effectiveness of utilizing additional semantic information for template learning.

\begin{figure}[t!] 
\centering
\includegraphics[width=0.5\textwidth]{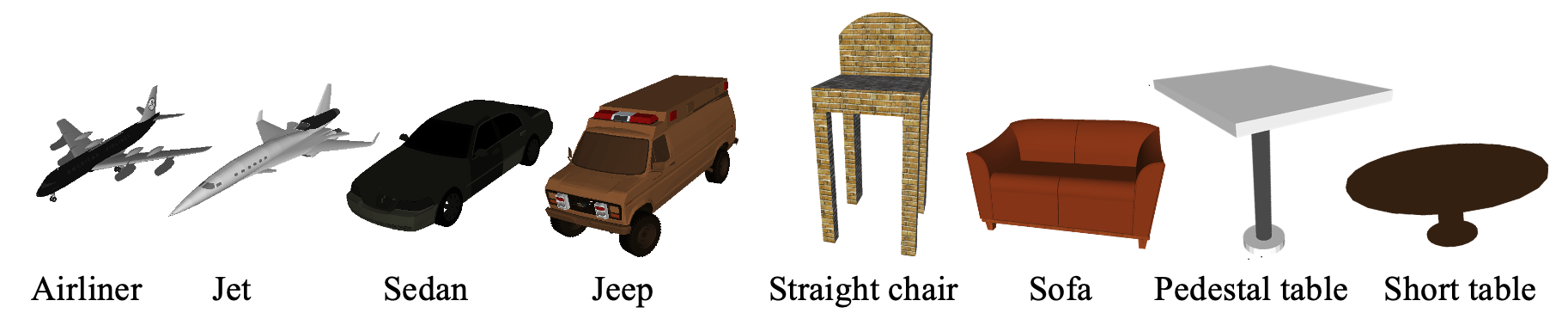}
\caption{\textbf{Types of source shapes for 2-shot part label transfer.}}
\label{fig:source}
\end{figure}

\begin{table}[ht!] 
	\caption{\textbf{2-shot part label transfer in subcategories of ShapeNet}. The reported numbers are overall mIoU over 5 runs.}
	\label{tab:subcategory}
    \centering
	\begin{tabular}{l c c}
		\toprule
            subcategories&DIF~\cite{deng2021deformed}& Ours\\
		\midrule
            airplane(airliner/jet) & 64.5{\scriptsize$\pm$1.61} &  \textbf{72.8}{\scriptsize$\pm$0.99} (\textcolor{blue}{+8.3$\uparrow$})\\
            car(sedan/jeep) & 65.1{\scriptsize$\pm$2.34} & \textbf{69.1}{\scriptsize$\pm$1.46} (\textcolor{blue}{+4.0$\uparrow$})\\
            chair(straight/sofa) & 66.3{\scriptsize$\pm$6.06} & \textbf{77.5}{\scriptsize$\pm$5.81} (\textcolor{blue}{+11.2$\uparrow$})\\
            table(pedestal/short) & 74.4{\scriptsize$\pm$10.1} & \textbf{79.7}{\scriptsize$\pm$6.26} (\textcolor{blue}{+5.3$\uparrow$})\\
            \bottomrule
	\end{tabular}
\end{table}

\begin{figure} 
\centering
\includegraphics[width=0.49\textwidth]{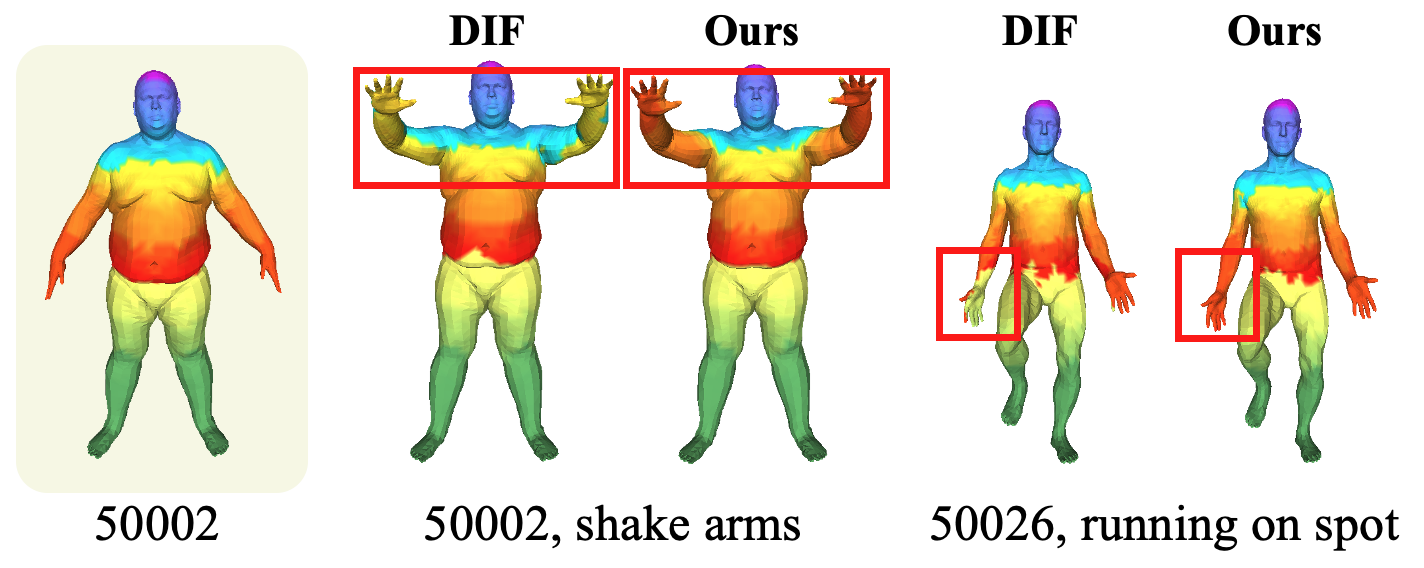}
\caption{\textbf{Color transfer on DFAUST~\cite{dfaust:CVPR:2017}.} We transfer different color labels over the up-axis from source shapes (colored in ivory) to target shapes and compare ours with DIF~\cite{deng2021deformed}.}
\label{fig:dfaust2}
\end{figure}

\subsection{Ablations}\label{sec:5.4}

\noindent\textbf{Ablation on regularizations.}
We conduct ablation studies regarding the proposed regularizations including part deformation consistency regularizations, global scale consistency regularization, and global deformation consistency regularization.
We conduct 5-shot part label transfer within a chair category, as same as \Cref{sec:5.1}.
Table~\ref{tab:abl} shows that each component of our proposed regularizers contributes to performance improvement for dense correspondence.
Note that by only using $\mathcal{L}_{\text{pdc}}$, \ie, using $\pdcgeo$ and $\pdcsem$, it still outperforms than baselines (80.4~\cite{zheng2020dit} and 82.1~\cite{deng2021deformed}).
The best performance was achieved by incorporating all proposed regularizations.

\begin{table}[t] 
  \centering
  \caption{\textbf{Ablations for every proposed regularization}. We report mIoU for 5-shot part label transfer on ShapeNet chair. Note that $\mathcal{L}_{\text{pdc}}$ denotes $\pdcgeo + \pdcsem$.}
  \label{tab:abl}
\setlength{\tabcolsep}{6pt}
\renewcommand{\arraystretch}{1.05}
  \begin{tabular}{c|ccc|c}
    \toprule
    \multicolumn{1}{c}{
    }&$\mathcal{L}_{\text{pdc}}$ &$\gdc$   &$\mathcal{L}_{\text{scale}}$ &chair\\
    \midrule
    (a)&\checkmark && &  83.4\\
    (b)&\checkmark  &\checkmark& &  83.9\\
    (c)&\checkmark  &\checkmark& \checkmark&  \textbf{84.2}\\
    \bottomrule
  \end{tabular}
\end{table}

\begin{table}[t]\footnotesize
  \centering
  \caption{\textbf{BAE-Net $k$ ablation.} We report mIoU for 5-shot part label transfer on ShapeNet chair over 3 runs.}
  \label{tab:kablation}
\setlength{\tabcolsep}{8pt}
\renewcommand{\arraystretch}{1}
  \resizebox{0.95\columnwidth}{!}{%
  \begin{tabular}{c|ccccc}
        \toprule
         & $k=2$& $k=4$ & $k=6$ & $k=8$ & $k=10$ \\
        \midrule
        chair& 83.8{\scriptsize$\pm$0.51}&83.5{\scriptsize$\pm$1.48}&82.8{\scriptsize$\pm$1.27} & 83.6{\scriptsize$\pm$0.85} &83.7{\scriptsize$\pm$1.33} \\
        \bottomrule
  \end{tabular}}
  \vspace{-10pt}
\end{table}

\noindent\textbf{BAE-Net~\cite{chen2019bae_net} with different $k$.}
We evaluate the effect of $k$ for BAE-Net~\cite{chen2019bae_net} with mean/standard deviation over 3 runs for ShapeNet~\cite{chang2015shapenet} chair in 
We train BAE-Net~\cite{chen2019bae_net} with a range of $k=(2,4,6,8,10)$, where $k=4$ is used in the main paper, and report mIoU for 5-shot part label transfer task.
Table~\ref{tab:kablation} shows the efficacy of semantic priors, where our framework achieves the average mIoU of 83.5, which shows higher performance than baselines.
More importantly, we observed that the performance is fairly robust over $k$.

\section{Conclusion}
We have proposed a new framework that learns a semantic-aware implicit template field. 
Our framework extracts semantic knowledge from a self-supervised feature extractor given unlabelled point clouds. 
By leveraging the semantic information, we propose novel conditioning with semantic-aware deformation code and part deformation consistency regularizations to encourage semantically plausible deformation.
Furthermore, we suggest global scaling consistency and global deformation consistency regularizations to enhance the learning process.
Through quantitative and qualitative evaluation, we demonstrate the superiority of our proposed method and verified the effect of all proposed components in our framework.

\paragraph*{Acknowledgments.}
This research was supported in part by the MSIT (Ministry of Science and ICT), Korea, under the ICT Creative Consilience program (IITP-2023-2020-0-01819) supervised by the IITP (Institute for Information \& communications Technology Planning \& Evaluation); the National Research Foundation of Korea (NRF) grant funded by the Korea government (MSIT) (NRF-2023R1A2C2005373); and Google Cloud Research Credits program with the award (GCLF-67WL-WU6D-MF14, V9P3-99XD-64CL-EH5H).

{\small
\bibliographystyle{unsrt}
\bibliography{main}
}
\clearpage

\begin{center}
    \Large \textbf{Supplementary materials}\\
\end{center}

\appendix
We provide additional experimental results/details and discussion in this supplement.
The supplement consists of 
(1) experiment details ($\eg$, dataset statistics, implementation details), 
(2) derivation of the closed form for the optimal global scaling factor $r$ given deformation $\mathcal{D}$,
(3) sensitivity test for different coefficients of proposed regularizations,
(4) additional results with different semantic priors,
(5-7) qualitative results on ShapeNetV2~\cite{chang2015shapenet}, ScanObjectNN~\cite{uy-scanobjectnn-iccv19}, and DFAUST~\cite{dfaust:CVPR:2017},
(8) analysis of learned implicit template fields, and
(9) limitations.
In each visualization, note that all shapes placed on the $\color{beige}{ivory}$ background (and leftmost) are source shapes for label transfer tasks.

\section{Experiment details}
\label{suppl_s1}

\subsection{Data statistics}
We mainly use four categories ($\eg$, chair, table, airplane and car) of ShapeNetV2\footnote{Copyright (c) 2022 ShapeNET.}~\cite{chang2015shapenet} with labels from ShapeNet-Part~\cite{shapenetprat/tog/YiKCSYSLHSG16} and KeypointNet~\cite{DBLP:conf/cvpr/YouLLCLMLW20}.
We follow DIF\footnote{https://github.com/microsoft/DIF-Net}~\cite{deng2021deformed} for data splitting and data preparation.
Table~\ref{tab:dstat} summarizes the data statistics. 
For additional experiments in Section 5.3, first, we use a subset of ShapeNetV2~\cite{chang2015shapenet} that consists of two subcategories in each category (\eg, airplane: airliner/jet, car: sedan/jeep, chair: straight, chair/sofa, table: pedestal table/short table).
As ShapeNetV2 contains incorrect subcategory labels, we manually cleaned subcategory labels and sample 50 shapes for each subcategory. To avoid any bias towards a specific shape structure, we use an equal number of subcategory shapes.
Second, for DFAUST~\cite{dfaust:CVPR:2017} dataset, we select two different subjects (50002 and 50026), where 50002 is of a larger size, being overweight and taller compared to 50026, and four distinct actions (\eg, shake arms, chicken wings, running on spot, and jumping jacks).
\begin{table}[hbt!]\footnotesize
  \caption{\textbf{Data statistics.}}\label{tab:dstat}
  \centering
\setlength{\tabcolsep}{6pt}
\renewcommand{\arraystretch}{1.05}
  \begin{tabular}{ccccc}
    \toprule
    \multicolumn{1}{c}{\multirow{2}{*}{}}&
    \multicolumn{4}{c}{Category} \\
    \cmidrule{2-5}
    &
     \multicolumn{1}{c}{\textbf{airplane}} & \multicolumn{1}{c}{\textbf{car}} & \multicolumn{1}{c}{\textbf{chair}} & \multicolumn{1}{c}{\textbf{table}}\\
    \midrule
    \midrule
   Training data& 3500&3000&4000&4000	 \\
   Part labeled data& 3195&	2728&	3551&3671\\
   Keypoint labeled data& 889&870&577&545\\
   \midrule
   Evaluation data (recon)& 100&	100&100&100	\\
    \bottomrule
  \end{tabular}
\vspace{-10pt}
\end{table}

\subsection{Implementation details}
For feature extraction, we mainly use BAE-Net\footnote{https://github.com/czq142857/BAE-NET}~\cite{chen2019bae_net} and follow their training schemes.
In detail, we prepare 8,192 surface points from voxel grid sampling and 32,768 query points from random sampling for each shape following IM-Net~\cite{chen2018implicit_decoder}.
We train the encoder with a batch size of 1 for 60 epochs, using Adam~\cite{kingma2015adam} optimizer with a learning rate of 0.0001 for car and 0.00005 for the rest of the categories.
It takes about 2 hours on average to fully train the encoder with a single GPU (RTX 2080ti).
For the hyperparameter $k$, we use 6/8/4/4 for airplane/car/chair/table category, and we use 256 for the global latent code dimension.

For training deformation field  $\mathcal{D}_{\theta_1}$ and template field $\mathcal{T}_{\theta_2}$, we train the model with a batch size
of 128 for 60 epochs, using Adam~\cite{kingma2015adam} optimizer with the learning rate of 0.0001.
It takes about 6 hours on average to fully train the model.
Also, we use 256 for the dimension of each shape latent code and we use 64 for the dimension of each part deformation prior.

For the proposed regularizations, we grid search coefficients such as global scale consistency $ \mathcal{L}_{\text{scale}}$ and global deformation consistency $\gdc$ in the range of [10,500]/[50,100].
For the most \textit{influential} regularization, which is part deformation consistency $\pdcgeo$, we grid search the coefficient in the separate range according to the categories ($\eg$, [1000,2000] for chair, [750,1000] for airplane and car, [250,500] for table).
The final coefficients we used are described in Table~\ref{hparam}.
For the rest of the regularizers, we fix the coefficient as 50 for $\pdcsem$, 100 for $\mathcal{L}_{\text{normal}}$, 1e6 for $\mathcal{L}_{\text{emb}}$ in every category.
For the regularizers for deformation smoothness $\mathcal{L}_{\text{smooth}}$, and minimal correction $\mathcal{L}_{\text{c}}$, we grid search in the range of [1,5,10]/[50,100,500] and apply 10/5/5/1 and 50/10/500/50 in sequential order to the four categories: airplane, car, chair, and table.

\begin{table}[hbt!]\footnotesize
  \caption{\textbf{Regularization coefficients for each category.}} \label{hparam}
  \centering
\setlength{\tabcolsep}{6pt}
\renewcommand{\arraystretch}{1.05}
  \begin{tabular}{c|ccc}
    \toprule
coef.&$\pdcgeo$&$\mathcal{L}_{\text{scale}}$&$\Lc_{\text{geo}}$\\
    \midrule
    \textbf{airplane}& 750&	10&	50 \\
    \textbf{car}& 1000	&10&	50\\
    \textbf{chair}& 2000&500&	100	\\
    \textbf{table}& 250&	500&	100\\
    \bottomrule
  \end{tabular}
\vspace{-10pt}
\end{table}

For baseline models, we validate surrogate tasks with provided pretrained models; if none exists, we train the model from scratch based on source code from the original authors, $\eg$, DIT\footnote{https://github.com/ZhengZerong/DeepImplicitTemplates} and AtlasNetV2\footnote{https://github.com/TheoDEPRELLE/AtlasNetV2}.
Unlike implicit template learning models, AtlasNetV2~\cite{atlasnetv2deprelle2019learning} does not learn a global template, rather it learns decomposed local patches.
Here, for AtlasNetV2~\cite{atlasnetv2deprelle2019learning}, we evaluate the part label transfer task as identical as DIF~\cite{deng2021deformed}, and also similarly evaluate the keypoint label transfer task, where we use the average points of each corresponding keypoint from source shapes as transferred keypoint labels.
Lastly, all experiments are implemented in Pytorch\footnote{Copyright (c) 2016-Facebook, Inc (Adam Paszke), Licensed under BSD-style license}~\cite{PyTorch} and Pytorch3D~\cite{ravi2020pytorch3d} and conducted on 4 NVIDIA RTX A6000.

\section{Proof of global scaling factor $r$}
\label{suppl_s2}

We propose global scale consistency regularization to preserve the scale of the implicit template field against strong deformations based on the following lemma and its proof.
\begin{lemma}
Given a scalar field (shape) $X \in \mathbb{R}^{3 \times M}$ and a non-rigid deformation $\mathcal{D}:\bx \in \mathbb{R}^3 \rightarrow \dfm \in \mathbb{R}^3$, we define a global scaling factor $r$ of $\mathcal{D}$ as an optimal solution to the following problem:
\begin{equation}
r^* = \argminU\limits_{r}\sum\limits_{i=1}^M \|\bx_i +\dfm_i - r\bx_i\|_2^2. 
\label{eq:scale_suppl}
\end{equation} 
Then, the optimal solution can be analytically obtained by
\begin{equation}
r^* =  \frac{\sum_{i=1}^M \bx_i^{\top} (\bx_i +\dfm_i)}{\sum_{j=1}^M (\bx_j^{\top} \bx_j)},
\label{eq:optr}
\end{equation} 
where $\bx_i \in X$ and $\dfm_i \in \mathcal{D}(x_i)$.
\end{lemma}

\begin{proof} 
We defined the global scaling factor as the optimal solution to the following problem: 
\begin{equation}
r^* = \argminU\limits_{r}\sum\limits_{i=1}^M \|\bx_i +\dfm_i - r\bx_i\|_2^2. 
\label{eq:scale_suppl}
\end{equation} 
To find a closed-form solution, we differentiate~\eqref{eq:scale_suppl} by $r$:
\begin{multline}
{\partial\over\partial r}\sum\limits_{i=1}^M \|\bx_i +\dfm_i - r\bx_i\|_2^2
= \sum\limits_{i=1}^M(\bx_i +\dfm_i - r\bx_i)^{\top}\bx_i \\
= \sum\limits_{i=1}^M(\bx_i +\dfm_i)^{\top}\bx_i - \sum\limits_{j=1}^Mr\bx_j^{\top}\bx_j    = 0 
\label{eq:scale_suppl2}
\end{multline} 
Finally, we can acquire the solution $r^*=\frac{\sum_{i=1}^M \bx_i^{\top} (\bx_i +\dfm_i)}{\sum_{j=1}^M\bx_j^{\top} \bx_j}$ from~\eqref{eq:scale_suppl2}.
\end{proof}
In this paper, we encourage the learned implicit template to preserve its scale by the \textit{average} scale of $N$ deformed shapes in a single batch.
That is, we simply estimate the global scaling of all $N$ shapes in a mini-batch by
$r_\text{batch} = \frac{\sum\limits_{s=1}^N\sum\limits_{i=1}^M \bx_i^{s^\top} (\bx_i^s +\dfm_i^s)}{\sum\limits_{s'=1}^N\sum\limits_{j=1}^M \bx_j^{s'^\top} \bx_j^{s'}}  \approx 1$.
This treats all shapes as a point cloud and finds one global scaling of it.

Further, we can impose a slightly different regularization with the expectation of $r_s$. Given a set of points $\{x_i^s\}_i$ in shape $s$,  a shape-specific global scaling factor $r_s$ and its expectation are defined as
\begin{equation}
\begin{split}
r_s &= \frac{\sum\limits_{i=1}^M \bx_i^{s^\top} (\bx_i^s +\dfm_i^s)}{\sum\limits_{j=1}^M \bx_j^{s^\top} \bx_j^s} \\
\mathbb{E}_s [r_s] &=\sum_{s=1}^N \frac{1}{N} r_s\\
\end{split}
\end{equation}
Then, the regularizer is given as
\begin{equation}
\mathcal{L}_\text{scale} = |\mathbb{E}[r] - 1|.
\end{equation}

We observed that in our preliminary experiments, these regularizations above allow more flexibility than enforcing the global scaling of each individual shape, \ie, $ \sum_{s=1}^N|r_s - 1|$. 
The final equation of global scaling consistency regularization is in ($\color{red}{9}$) of the main paper.

\section{Sensivity test for different coefficients}
\label{suppl_s3_0}

We analyze the effect of coefficients for suggested regularizations ($\pdcgeo/ \gdc/ \mathcal{L}_{\text{scale}}$).
We perform experiments with a sub-dataset (same as Table 5 in the main paper) for ShapeNet airplane/chair and report mIoU as 2-shot part label transfer results.
Figure~\ref{sensitivity} shows that $\pdcgeo$ (highlighted with the \textcolor{red}{red} box for cases without $\pdcgeo$) significantly improves performance. 
When the coefficient of $\pdcgeo$ is properly set, our method is robust to the choice of weights for other losses: $\gdc$ and $\mathcal{L}_{\text{scale}}$.
As shown in the boxes highlighted in \textcolor{amber}{yellow} and \textcolor{green2}{green}, our method stably achieves good performance regardless of the coefficients for $\gdc$ and $\mathcal{L}_{\text{scale}}$.

\begin{figure}[h] 
\centering
  \vspace{-10pt}
\includegraphics[width=0.5\textwidth]{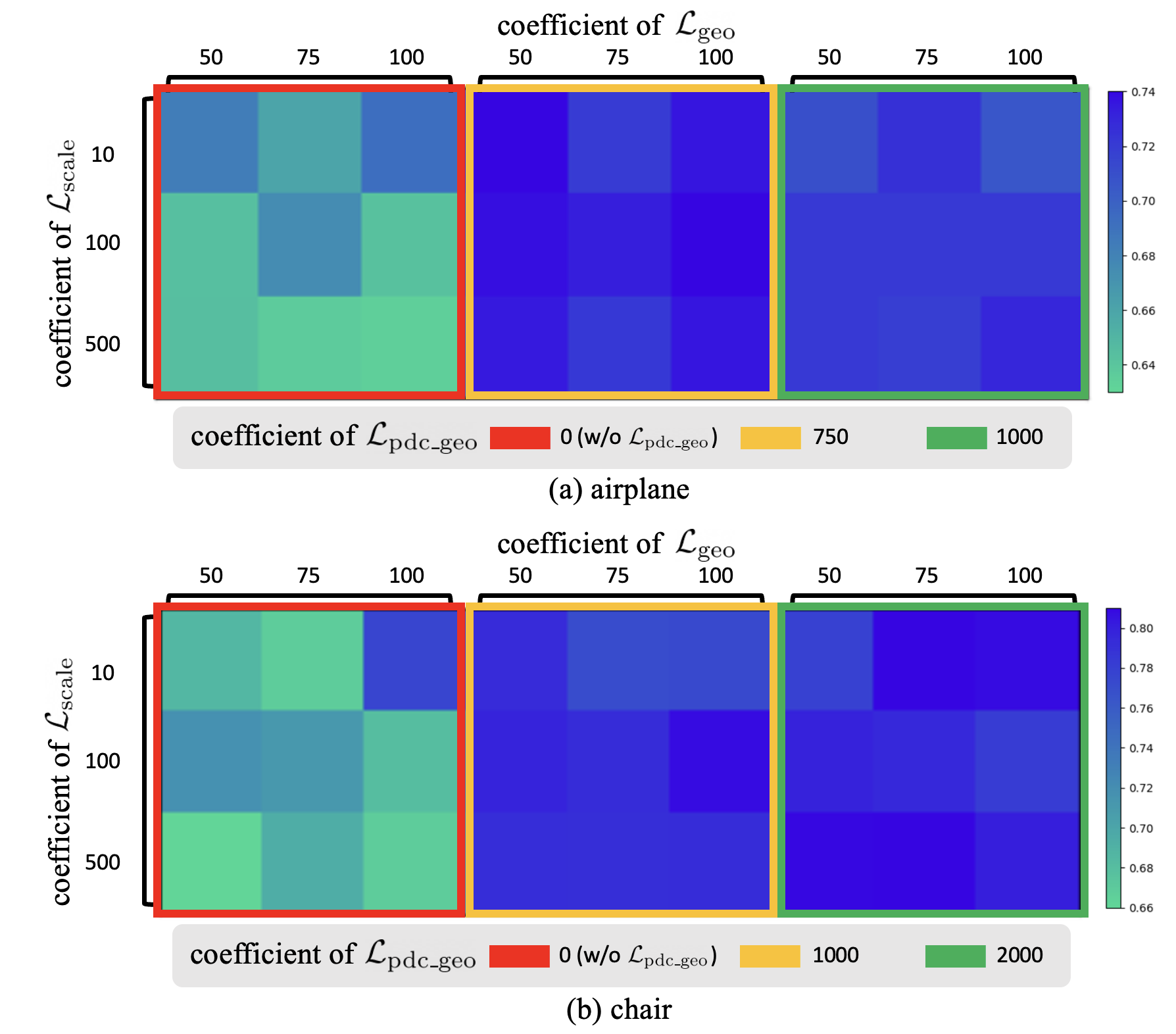}
\caption{
\label{sensitivity}
\small\textbf{Sensitivity test for $\pdcgeo/ \gdc/ \mathcal{L}_{\text{scale}}$.} }
\label{fig:sense}
\end{figure}

\section{Different semantic priors}
\label{suppl_s3}

We provide an additional experimental result with RIM-Net~\cite{niu2022rim}, which is a self-supervised co-segmentation model for 3D object shapes. 
The pre-trained RIM-Net\footnote{https://github.com/chengjieniu/RIM-Net} is used and we conduct 2-shot part label transfer with the subcategories for the chair of ShapeNetV2.
Since RIM-Net provides different levels of part semantics, \eg, two-part partitions for level 1, and eight-part partitions for level 3, we leverage this characteristic to evaluate our framework across different levels of part quality.
Based on Table~\ref{priors}, we believe that even if the given prior does not have high-level semantics, our framework improves performance as long as the prior is \textit{consistent}.

\begin{table}[h!]
  \centering
  \caption{\textbf{Utilizing different semantic priors}. 2-shot part label transfer in subcategories (straight chair and sofa) for ShapeNet chair.}
  \label{priors}
  \resizebox{0.8\columnwidth}{!}{%
  \Huge{
  \begin{tabular}{c|ccc||c}
    \toprule
       & Ours-RIM(lv1) & Ours-RIM(lv3)&Ours-BAE & DIF\\
    \midrule
    chair &70.1 & 78.3 &80.7&67.2\\
    \bottomrule
  \end{tabular}
  }}
\end{table}

\begin{figure}[t!]
\centering
            \includegraphics[width=0.4\textwidth]{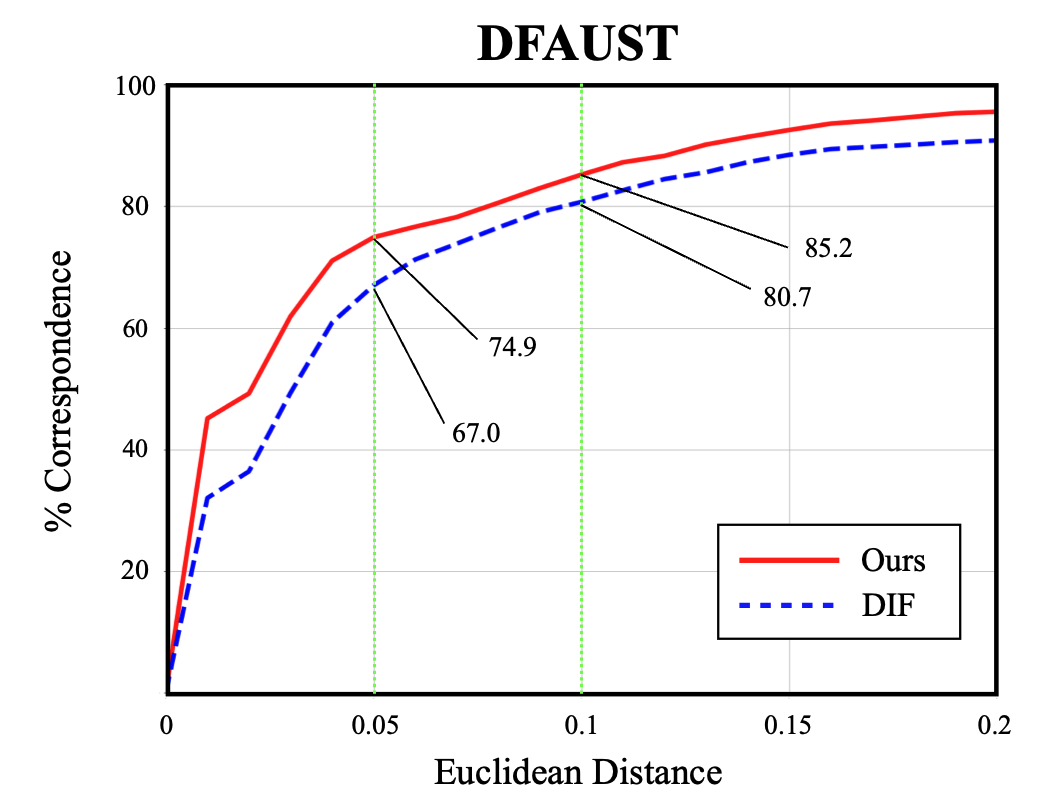}
        \caption
         {\textbf{Keypoint transfer performance for DFAUST~\cite{dfaust:CVPR:2017}.} We measure correspondence accuracy (PCK score).}
\label{fig:dfaust-table}
\end{figure}

\section{Qualitative results on ShapeNet}
\label{suppl_s4}

We provide additional visualizations as in Figure~\ref{fig:snkp}, Figure~\ref{fig:snpl}, and~\ref{fig:sntt} to show keypoint/part label/texture transfer results on ShapeNetV2~\cite{chang2015shapenet}. 
These results illustrate the importance of imposing semantics in implicit template learning for generic object shapes.

\section{Keypoint transfer on ScanObjectNN}
\label{suppl_s5}

In Figure~\ref{fig:sonn}, we conduct qualitative analysis via keypoint transfer task on ScanObjectNN~\cite{uy-scanobjectnn-iccv19} to validate the robustness of our method even with the domain gap between synthetic data and real-world data.
We transfer the keypoint labels of chair/table categories in ShapeNetV2~\cite{chang2015shapenet} to corresponding shapes in ScanObjectNN~\cite{uy-scanobjectnn-iccv19}.
Since shapes in ScanObjectNN are real-world scanned data, \ie, they are not always watertight, we only use scanned surface points for inference.
Our framework shows superior performance over baseline models, supporting the importance of understanding semantics for shape correspondence.
Visualizations of the chair in Figure~\ref{fig:sonn} are clear examples.
Two keypoints of pea-green and orange are geometrically close but semantically different in the source chair. 
After transferring keypoints, our framework is the only model that enables disentangling two keypoints in the target chair, where arms and legs are separated, unlike the source chair.
\begin{figure}[t!] 
\centering
\includegraphics[width=0.36\textwidth]{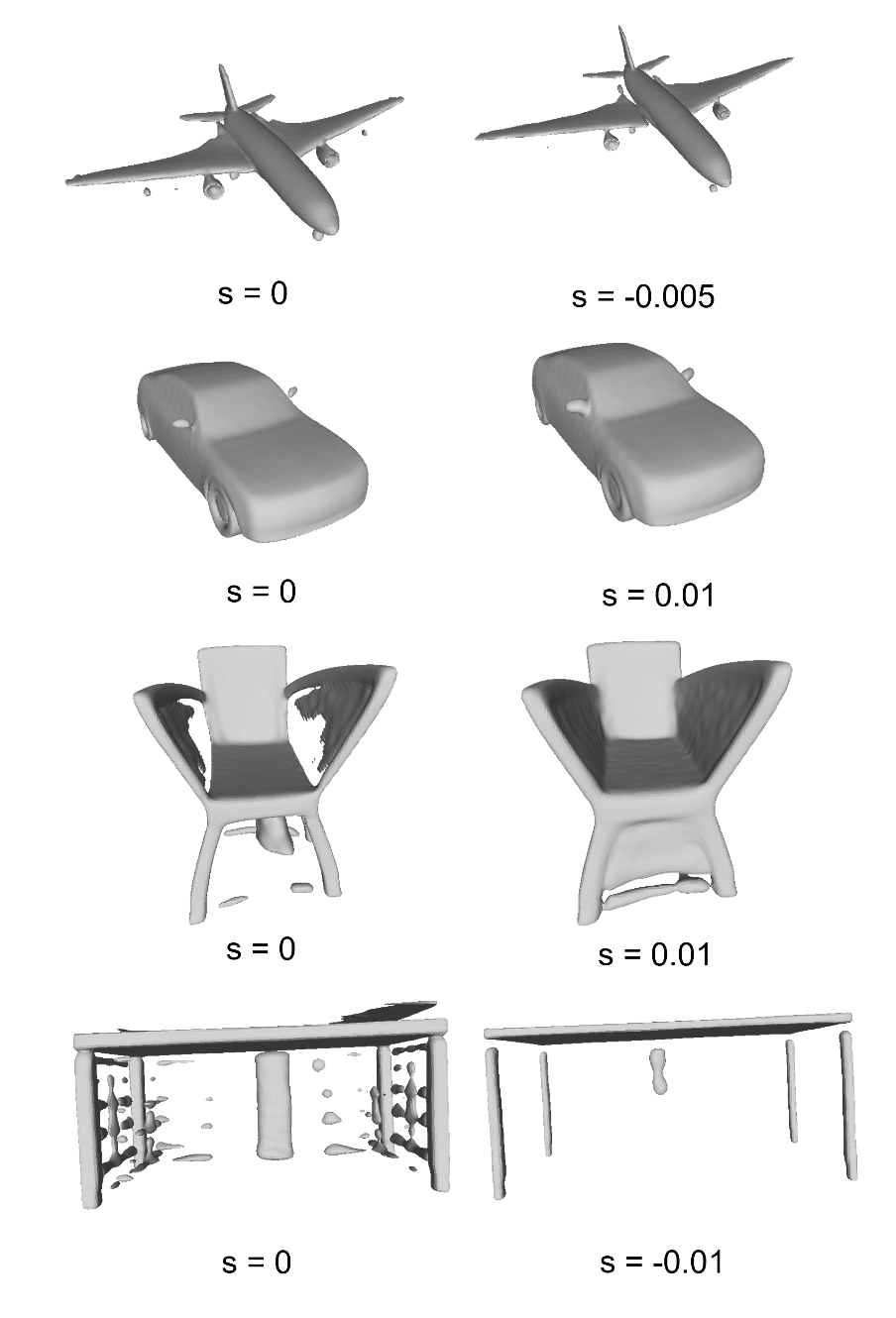}
\caption{\textbf{Visualization of learned implicit template fields for four categories by extracting different iso-surfaces.}}
\label{fig:tmpl}

\end{figure}

\section{Keypoint transfer on DFAUST}
\label{suppl_s6}

We provide detailed keypoint transfer results to evaluate unsupervised correspondence performance on non-rigid shapes~\cite{dfaust:CVPR:2017} in Figure~\ref{fig:dfaust-table} and Figure~\ref{fig:dfkp}. 
For evaluation, we assume that we do not know the ground truth correspondence between human shapes.
Although learning a suitable global template for various actions with large local deformation scales is a challenging task, our method consistently shows better and more consistent correspondence accuracy compared to the baseline, as shown in the following figures.
In particular, we present the full keypoint transfer result in Figure~\ref{fig:dfaust-table}, where we use source shape/selected keypoints as in the shape in Figure~\ref{fig:dfkp} (leftmost shape in beige).
The performance is PCK scores given a continuous range of thresholds.
We observe our framework (red) consistently outperforms the baseline (blue), indicating that learned correspondences from our framework are more accurate.
Visualizations in Figure~\ref{fig:dfkp} also support our framework, \eg, our model transfers “hand” keypoints (highlighted in blue balls) while the baseline transfers them to the waist or the elbow.

\section{Analysis of learned implicit template fields}
\label{suppl_s7}
Figure~\ref{fig:tmpl} illustrates the iso-surfaces of four categories (\eg, airplane, car, chair, table) extracted from learned implicit template fields.
Our framework learns a template by imposing semantically consistent mapping, rather than learning realistic templates.
We observe that the iso-surface of the learned templates sometimes has stretched parts in more challenging categories, which is beneficial for semantically mapping shapes with high variability.
Thus, it leads to improved performance on dense correspondence.

\section{Limitations}
Since we utilize self-supervised segmentation models for knowledge distillation, the performance can be highly dependent on the part segmentation quality of the feature extractor.
However, we have demonstrated that our framework consistently improves correspondence performance with various semantic priors, even with coarse part semantics such as BAE-Net~\cite{chen2019bae_net} with $k=2$ or level 1 RIM-Net~\cite{niu2022rim}.
Note that, unlike 2D domain, there are no powerful self-supervised segmentation models available yet, such as DINO~\cite{caron2021emerging}. 
If such models emerge for 3D domain, our framework can potentially achieve even better performance. 
These are left for future works.

\begin{figure*}[t!]
\centering
            \includegraphics[width=0.75\textwidth]{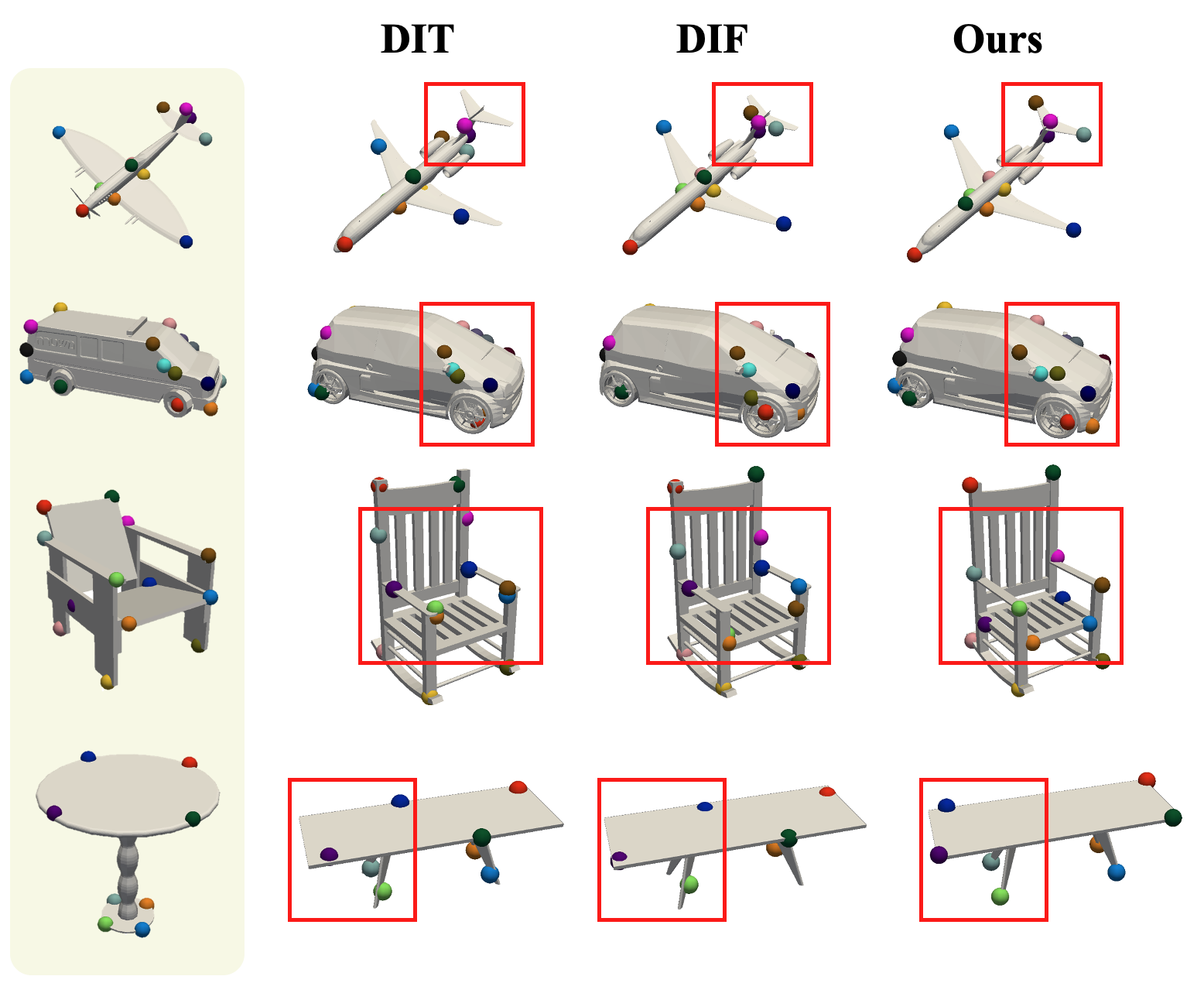}
        \caption
         {\textbf{Additional comparison on keypoint transfer in ShapeNetV2.}}
\label{fig:snkp}
\vspace{-10pt}
\end{figure*}
\begin{figure*}[t!]
\centering
            \includegraphics[width=0.76\textwidth]{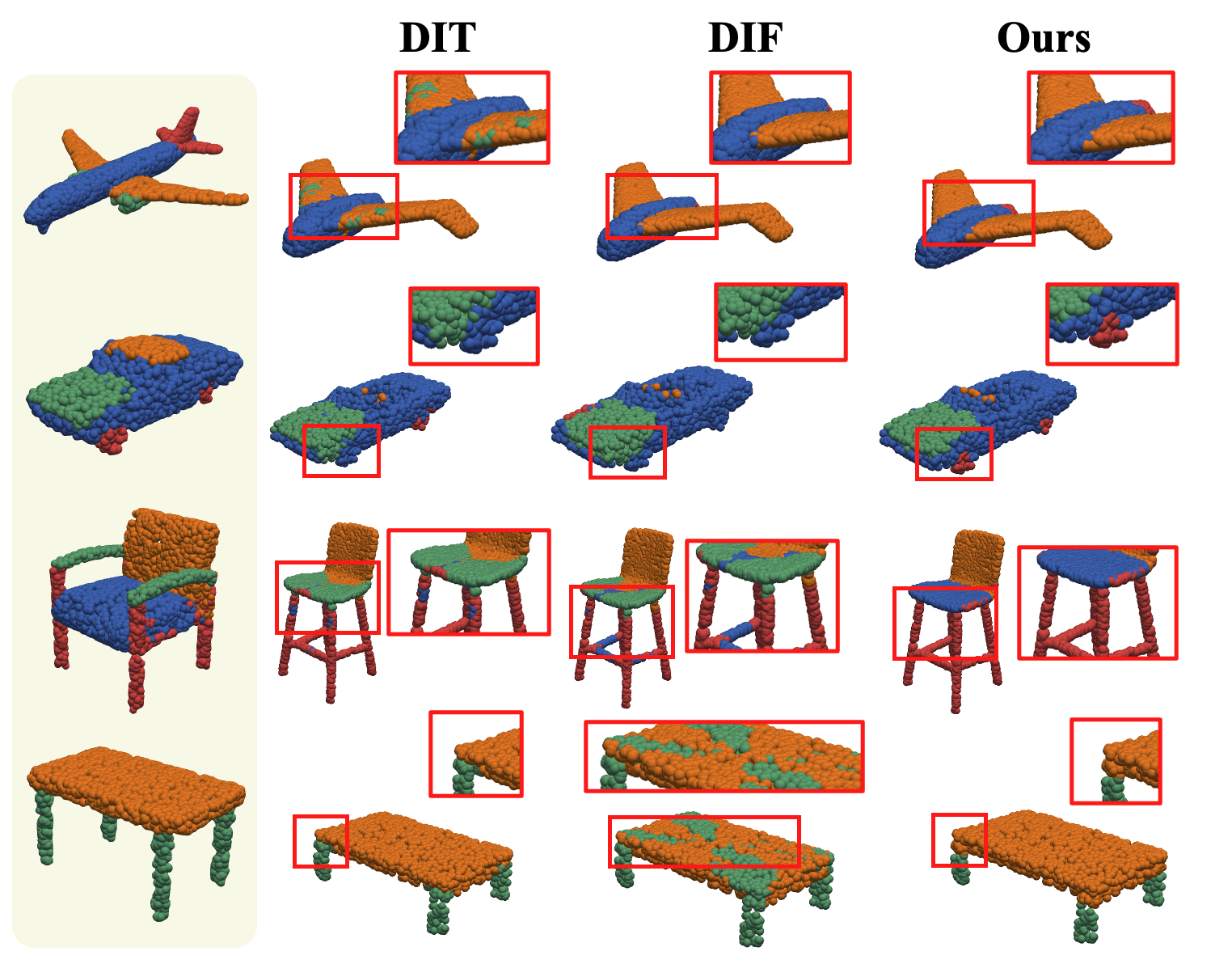}
        \caption
         {\textbf{Additional comparison on part label transfer ShapeNetV2.}}
\label{fig:snpl}
\vspace{-10pt}
\end{figure*}

\begin{figure*}[t!]
\centering
            \includegraphics[width=0.98\textwidth]{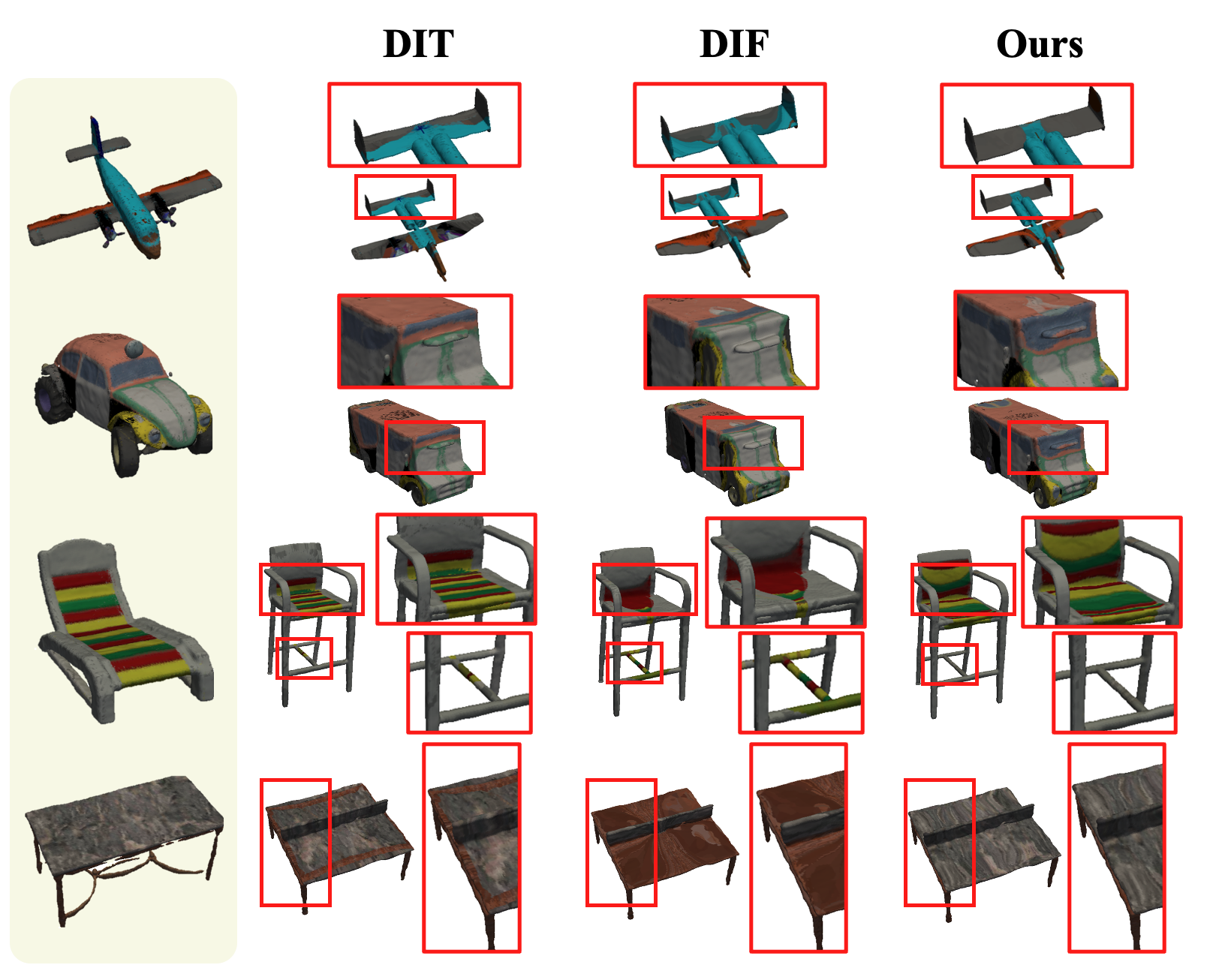}
        \caption
         {\textbf{Additional comparison on texture transfer ShapeNetV2.}}
\label{fig:sntt}
\end{figure*}

\begin{figure*}[t!]
\centering
            \includegraphics[width=0.88\textwidth]{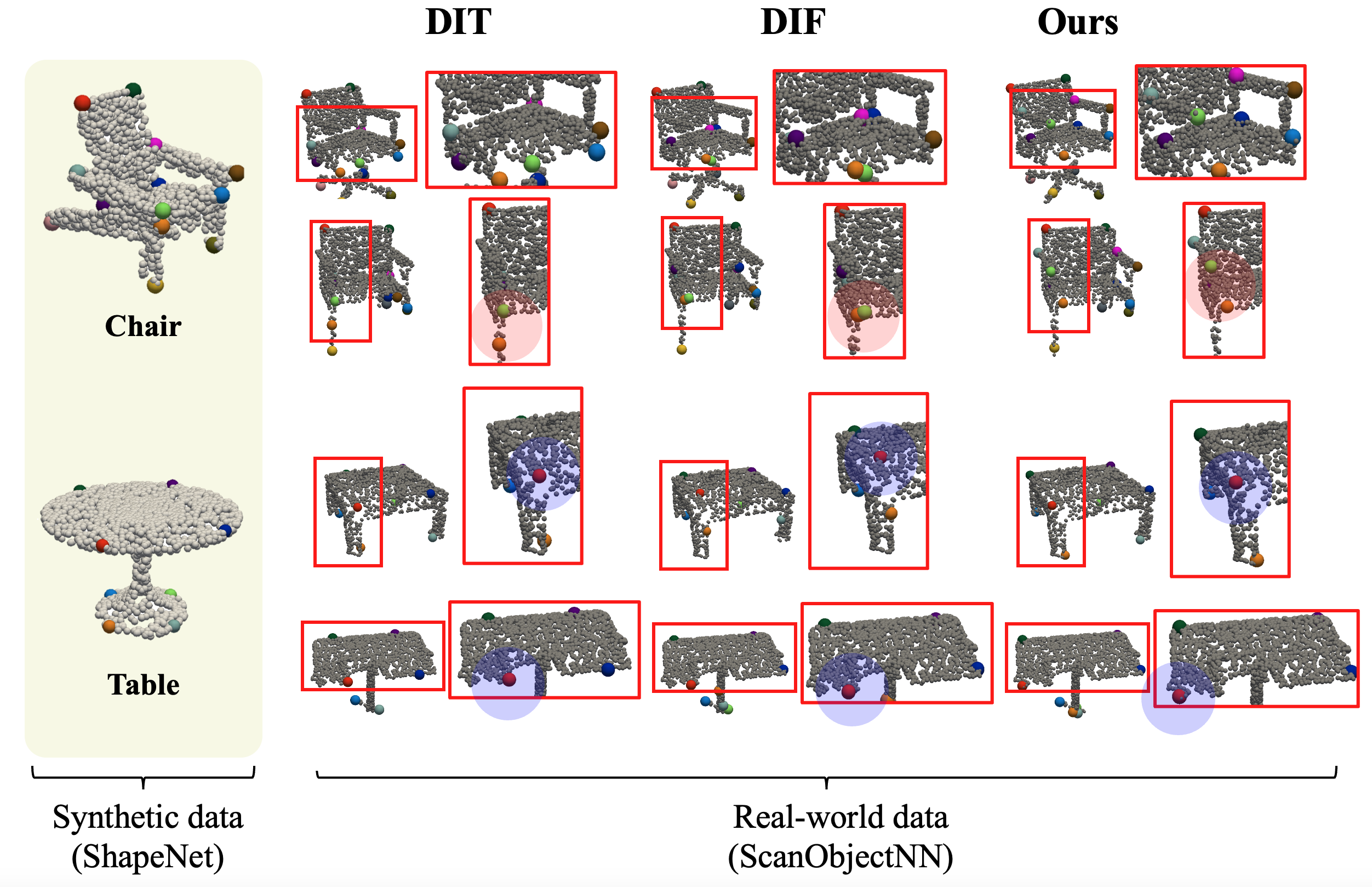}
        \caption
         {\textbf{Comparison on keypoint transfer in ScanObjectNN~\cite{uy-scanobjectnn-iccv19}.} We transfer keypoint from synthetic data (ShapeNetV2) to real-world scanned data (ScanObjectNN) to validate the robustness towards the domain gap (zoom-in for better visualization).}
\label{fig:sonn}
\vspace{-10pt}
\end{figure*}

\begin{figure*}[t!]
\centering
            \includegraphics[width=0.8\textwidth]{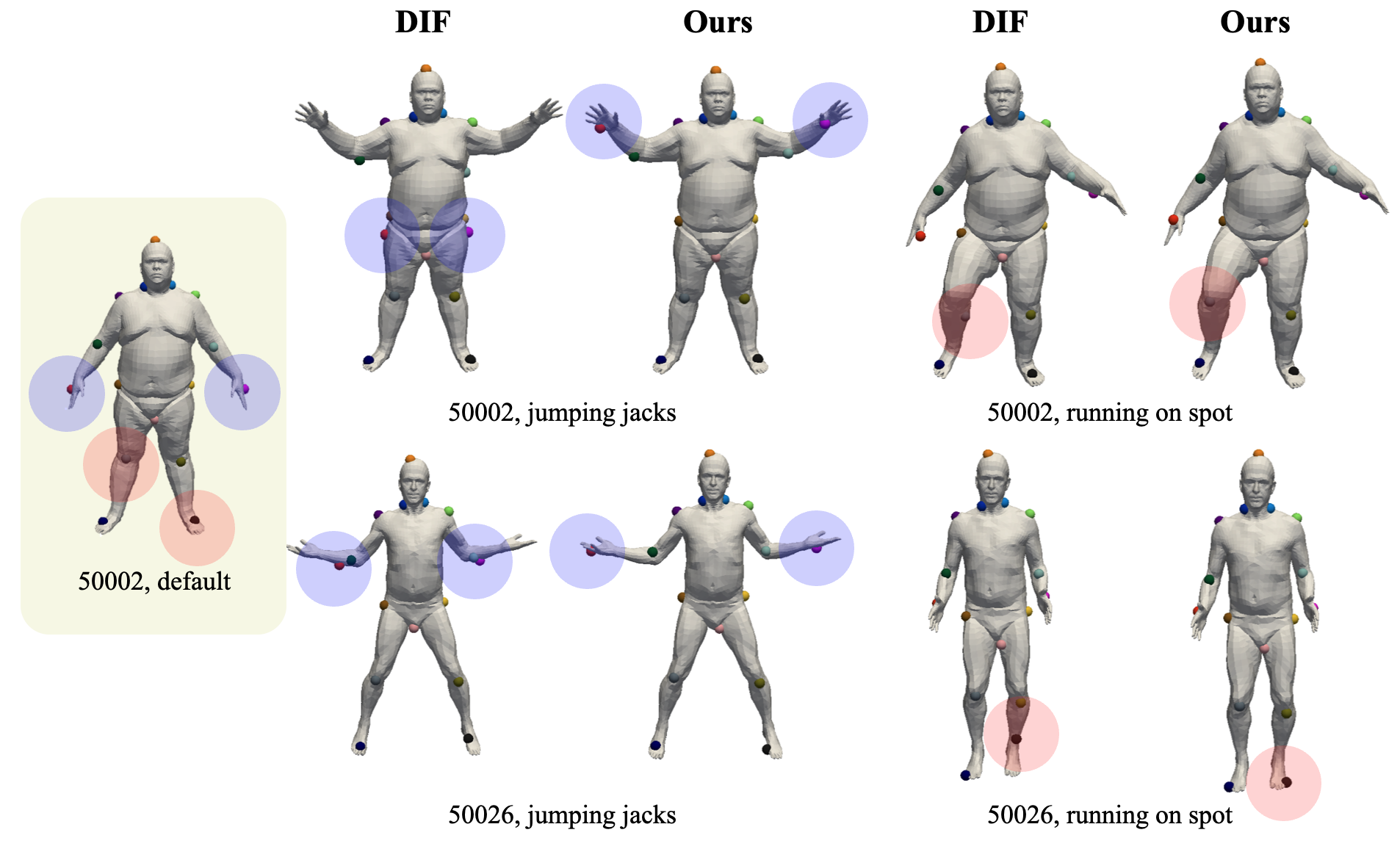}
        \caption
         {\textbf{Comparison on keypoint transfer in DFAUST~\cite{dfaust:CVPR:2017}.} Blue balls indicate keypoints for both hands, while red balls indicate keypoints for the right knee and the left foot (zoom-in for better visualization).}
\label{fig:dfkp}
\end{figure*}

\end{document}